\documentclass[twoside,11pt]{article}

\usepackage{jmlr2e, amsmath, color, booktabs,enumerate, moresize, dsfont, algorithm, algorithmic}

\newcommand{\reals}{I\!\!R} 

\usepackage[colorlinks=false,allbordercolors={1 1 1}]{hyperref}

\DeclareMathOperator*{\argmin}{arg\,min}

\newtheorem{dfn}{Definition}
\newtheorem{clm}{Claim}

\jmlrheading{16}{2015}{2859-2900}{5/14; Revised 3/15}{12/15}{John P. Cunningham and Zoubin Ghahramani}

\ShortHeadings{Linear Dimensionality Reduction}{Cunningham and Ghahramani}
\firstpageno{1}

\begin{document}

\title{Linear Dimensionality Reduction: \\Survey, Insights, and Generalizations}

\author{%
\name John P. Cunningham
\email jpc2181@columbia.edu \\
\addr Department of Statistics\\
Columbia University\\
New York City, USA\\
\AND
\name Zoubin Ghahramani
\email{zoubin@eng.cam.ac.uk} \\
\addr Department of Engineering\\
University of Cambridge\\
Cambridge, UK
}

\editor{Gert Lanckriet}

\maketitle

\begin{abstract}
Linear dimensionality reduction methods are a cornerstone of analyzing high dimensional data, due to their simple geometric interpretations and typically attractive computational properties.
These methods capture many data features of interest, such as covariance, dynamical structure, correlation between data sets, input-output relationships, and margin between data classes.  Methods have been developed with a variety of names and motivations in many fields, and perhaps as a result the connections between all these methods have not been highlighted.  
Here we survey methods from this disparate literature as optimization programs over matrix manifolds.  We discuss principal component analysis, factor analysis, linear multidimensional scaling, Fisher's linear discriminant analysis, canonical correlations analysis, maximum autocorrelation factors, slow feature analysis, sufficient dimensionality reduction, undercomplete independent component analysis, linear regression, distance metric learning, and more. 
This optimization framework gives insight to some rarely discussed shortcomings of well-known methods, such as the suboptimality of certain eigenvector solutions.   Modern techniques for optimization over matrix manifolds enable a  generic linear dimensionality reduction solver, which accepts as input data and an objective to be optimized, and returns, as output, an optimal low-dimensional projection of the data.  This simple optimization framework further allows straightforward generalizations and novel variants of classical methods, which we demonstrate here by creating an orthogonal-projection canonical correlations analysis.
More broadly, this survey and generic solver suggest that linear dimensionality reduction  can move toward becoming a blackbox, objective-agnostic numerical technology.  
\end{abstract}


\bigskip
\section{Introduction}

Linear dimensionality reduction methods have been developed throughout statistics, machine learning, and applied fields for over a century, and these methods have become indispensable tools for analyzing high dimensional, noisy data.  
These methods produce a low-dimensional linear mapping of the original high-dimensional data that preserves some feature of interest in the data.  Accordingly, linear dimensionality reduction can be used for visualizing or exploring structure in data, denoising or compressing data, extracting meaningful feature spaces, and more.  
 This abundance of methods, across a variety of data types and fields, suggests a great complexity to the space of linear dimensionality reduction techniques.  
As such, there has been little effort to consolidate our understanding.  Here we survey a host of methods and investigate when a more general optimization framework can improve performance and extend the generality of these techniques.

We begin by defining linear dimensionality reduction (Section \ref{sec:def}), giving a few canonical examples to clarify the definition.  We then interpret linear dimensionality reduction in a simple optimization framework as a program with a problem-specific objective over orthogonal or unconstrained matrices.  Section \ref{sec:review} surveys principal component analysis (PCA; \citealp{pearson1901, eckart1936}), multidimensional scaling (MDS; \citealp{torgerson1952, coxBook, borgBook}), Fisher's linear discriminant analysis (LDA; \citealp{fisher1936,rao1948}), canonical correlations analysis (CCA; \citealp{hotelling1936}), maximum autocorrelation factors (MAF;  \citealp{switzer1984}), slow feature analysis (SFA; \citealp{wiskott2002, wiskott2003}), sufficient dimensionality reduction (SDR; \citealp{fukumizu2004dimensionality, adragni2009}), locality preserving projections (LPP; \citealp{niyogi2004locality,he2005neighborhood}), undercomplete independent component analysis (ICA; e.g. \citealp{icaBook}), linear regression, distance metric learning (DML; \citealp{kulis2012metric, yang2006distance}), probabilistic PCA (PPCA; \citealp{tipping99, roweisSPCA, theobald75}), factor analysis (FA; \citealp{spearman1904}), several related methods, and important extensions such as kernel mappings and regularizations.  

A common misconception is that many or all linear dimensionality reduction problems can be reduced to eigenvalue or generalized eigenvalue problems.  Not only is this untrue in general, but it is also untrue for some very well-known algorithms that are typically thought of as generalized eigenvalue problems.  The suboptimality of using eigenvector bases in these settings is rarely discussed and is one notable insight of this survey.  Perhaps inherited from this eigenvalue misconception, a second common tendency is for practitioners to greedily choose the low-dimensional data: the first dimension is chosen to optimize the problem objective, and then subsequent dimensions are chosen to optimize the objective on a residual or reduced data set.  
The optimization framework herein shows the limitation of this view.  More importantly, the framework also suggests a more generalized linear dimensionality reduction solver that encompasses all eigenvalue problems as well as many other important variants.  In this survey we restate these algorithms as optimization programs over matrix manifolds that have a well understood geometry and a well developed optimization literature \cite[]{absilBook}.  This simple perspective leads to a generic algorithm for linear dimensionality reduction, suggesting that, like numerical optimization more generally, linear dimensionality reduction can become abstracted as a numerical technology for a range of problem-specific objectives.
In all, this work: \emph{(i)} surveys the literature on linear dimensionality reduction, \emph{(ii)} gives insights to some rarely discussed shortcomings of traditional approaches, and \emph{(iii)} provides a simple algorithmic template for generalizing to many more problem-specific techniques.

\section{Linear Dimensionality Reduction as a Matrix Optimization Program}
\label{sec:def}

We define linear dimensionality reduction as all methods with the problem statement:  

\begin{dfn}[Linear Dimensionality Reduction]
\label{dfn:ldr}
Given $n$ $d$-dimensional data points $X = [x_1,...,x_n] \in \reals^{d \times n}$ and a choice of dimensionality $r < d$, optimize some objective $f_X(\cdot)$ to produce a linear transformation $P \in \reals^{r \times d}$, and call $Y = PX \in \reals^{r \times n}$ the low-dimensional transformed data. 
\end{dfn}

Note that throughout this work we assume without loss of generality that data $X$ is mean-centered, namely $X1 = 0$.  To make this definition concrete, we briefly detail two widespread linear dimensionality reduction techniques: principal component analysis  (PCA; \citealp{pearson1901}) and canonical correlations analysis (CCA; \citealp{hotelling1936}). 
PCA maximizes data variance captured by the low-dimensional projection, or equivalently minimizes the reconstruction error (under the $\ell_2$-norm) of the projected data points with the original data, namely
\begin{equation*}
f_X(M) = || X - MM^\top X||^2_F.
\end{equation*}
\noindent  Here $M$ is a matrix with $r$ orthonormal columns.  In the context of Definition \ref{dfn:ldr}, optimizing $f_X(M)$ produces an $M$ such that $P = M^\top$, and the desired low-dimensional projection is $Y = M^\top X$.   PCA is discussed in depth in Section \ref{sec:pca}.

We stress that the notation of $M$ and $P$ in Definition \ref{dfn:ldr} is not redundant, but rather is required for other linear dimensionality reduction techniques where the linear transformation $P$ does not equal the optimization variable $M$ (as it does in PCA).  Consider CCA, another classical linear dimensionality reduction technique that jointly maps two data sets $X_a  \in \reals^{d_a \times n}$ and $X_b  \in \reals^{d_b \times n}$ to $Y_a  \in \reals^{r \times n}$ and $Y_b  \in \reals^{r \times n}$, such that the sample correlation between $Y_a$ and $Y_b$ is maximized\footnote{As a point of technical detail, note that the use of two data sets and mappings is only a notational convenience;  writing the CCA projection as $Y = \begin{bmatrix}Y_a \\ Y_b\end{bmatrix} = \begin{bmatrix}P_a & 0 \\ 0 & P_b\end{bmatrix}\begin{bmatrix}X_a \\ X_b\end{bmatrix} = PX$, we see that CCA adheres precisely to Definition \ref{dfn:ldr}.}.  Under the additional constraints that $Y_a$ and $Y_b$ have uncorrelated variables ($Y_a^{}Y_b^\top = \Lambda$, a diagonal matrix) and be individually uncorrelated with unit variance ($\frac{1}{n}Y_a^{}Y_a^\top = \frac{1}{n}Y_b^{}Y_b^\top = I$), a series of standard steps produces the well known objective
\begin{equation*}
f_X\left( M_a , M_b \right) = \frac{1}{r} \mathrm{ tr}\left( M_a^\top (X_a^{} X_a^\top)^{-1/2}X_a^{}X_b^\top (X_b^{} X_b^\top)^{-1/2} M_b^{} \right), 
\end{equation*}
\noindent as will be detailed in depth in Section \ref{sec:cca}.  This objective is maximized when $M_a^\top$ and $M_b^\top$ are the left and right  singular vectors of the matrix $(X_a^{} X_a^\top)^{-1/2}X_a^{}X_b^\top (X_b^{} X_b^\top)^{-1/2}$.  In the context of Definition \ref{dfn:ldr}, the low dimensional canonical variables $Y_a$ are then related to the original data as $Y_a = P_aX_a \in \reals^{r \times n}$, where $P_a = M_a^\top \left(X_a^{}X_a^\top\right)^{-1/2}$ (and similar for $Y_b$).   Since $M_a$ has by definition orthonormal columns, CCA, by inclusion of the whitening term $\left(X_a^{}X_a^\top\right)^{-1/2}$, does not represent an orthogonal projection of the data.
Accordingly, CCA and PCA point out two key features of linear dimensionality reduction and Definition \ref{dfn:ldr}:  first, that the objective function $f_X(\cdot)$ need not entirely define the linear mapping $P$ to the low-dimensional space; and second, that not all linear dimensionality reduction methods need be orthogonal projections, or indeed projections at all.     

Note also that both PCA and CCA result in a matrix decomposition, and indeed a common approach to many linear dimensionality reduction methods is to attempt to cast the problem as an eigenvalue or generalized eigenvalue problem \cite[]{burges2010}.  This pursuit can be fruitful but is limited, often resulting in ad hoc or suboptimal algorithms.   As a specific example, in many settings orthogonal projections of data are required for visualization and other basic needs.  Can we create an \emph{Orthogonal CCA}, where we seek orthogonal projections $Y_a = M_a^\top X_a^{}$ for a matrix $M_a$ with orthonormal columns (and similar for $Y_b$), such that the sample correlation between $Y_a$ and $Y_b$ is maximized?  No known eigenvalue problem can produce this projection, so one tempting and common approach is to orthonormalize $P_a$ and $P_b$ (the results found by traditional CCA).  We will show that this choice can be significantly suboptimal, and in later sections we will create Orthogonal CCA using a generic optimization program.  Thus matrix decomposition approaches suggest an unfortunate limitation to the set of possible linear dimensionality reduction problems, and a broader framework is required to fully capture Definition \ref{dfn:ldr} and linear dimensionality reduction. 
 
\subsection{Optimization Framework for Linear Dimensionality Reduction}
\label{sec:opt}

All linear dimensionality reduction methods presented here
can be viewed as solving an optimization program over a matrix manifold $\mathcal{M}$, namely
\begin{equation}
\label{eq:1}
\begin{aligned}
& \text{minimize}
& & f_X(M) \\
& \text{subject to}
&& M \in \mathcal{M}.
\end{aligned}
\end{equation}
\noindent Given Definition \ref{dfn:ldr}, the intuition behind this optimization program should be apparent:  the objective $f_X(\cdot)$ defines the feature of interest to be captured in the data, and the matrix manifold encodes some aspects of the linear mapping $P$ such that $Y = PX$.\footnote{Note that several methods will require optimization over additional auxiliary unconstrained variables, which can be addressed algorithmically via a coordinate descent approach (alternating optimizations over the auxiliary variable and Equation \ref{eq:1}) or some more nuanced scheme.}

All methods considered here specify $M$ as one of two matrix forms.  First, some methods are unconstrained optimizations over rank $r$ linear mappings, implying the trivial manifold constraint of Euclidean space, which we denote as $M \in \reals^{d \times r}$.   In this case, optimization may be straightforward, and algorithms like expectation-maximization \cite[]{dempster77} or standard first order solvers have been well used.  

Second, very often the matrix form will have an orthogonality constraint $\mathcal{M} = \{M \in \reals^{d \times r} ~:~ M^\top M = I\}$, corresponding to orthogonal projections of the data $X$.  In this case we write $\mathcal{M} = \mathcal{O}^{d \times r}$.
As noted previously, the typical and often flawed approach is to attempt to cast these problems as eigenvalue problems.  
Instead, viewed through the lens of Equation \ref{eq:1}, linear dimensionality reduction is simply an optimization program over a matrix manifold,  and indeed there is a  well-developed optimization literature for matrix manifolds (foundations include  \citealp{luenberger1972, gabay1982, edelman1998}; an excellent summary is \citealp{absilBook}).

As a primary purpose of this work is to survey linear dimensionality reduction, we first detail linear dimensionality reduction techniques using this optimization framework.   We then implement a generic solver for programs of the form Equation \ref{eq:1}, where $\mathcal{M}$ is the family of orthogonal matrices $\mathcal{O}^{d \times r}$.   Thus we show the framework of Equation \ref{eq:1} to be not only conceptually simplifying, but also algorithmically simplifying.  Instead of resorting to ad hoc (and often suboptimal) formulations for each new problem in linear dimensionality reduction, practitioners need only specify the objective $f_X(\cdot)$ and the high-dimensional data $X$, and these numerical technologies can produce the desired low-dimensional data. Section \ref{sec:examples} validates this claim by applying this generic solver  without change to different objectives $f_X(\cdot)$, both classic and novel.
 We require only the condition that $f_X(\cdot)$ be differentiable in $M$ to enable simple gradient descent methods.  However, this choice is a convenience of implementation and not a fundamental issue, and thus approaches for optimization of nondifferentiable objectives over nonconvex sets (here  $\mathcal{O}^{d \times r}$) could be readily introduced to remove this restriction (for example, \citealp{boyd2011distributed}).

\section{Survey of Linear Dimensionality Reduction Methods}
\label{sec:review}

We now review linear dimensionality reduction techniques using the framework of Section \ref{sec:def}, to understand the problem-specific objective and manifold constraint of each method.

\subsection{Linear Dimensionality Reduction with Orthogonal Matrix Constraints}
\label{sec:orth}

Amongst all dimensionality reduction methods, the most widely used techniques are orthogonal projections.  These methods owe their popularity in part due to their simple geometric interpretation as a low-dimensional view of high-dimensional data.  This interpretation is of great comfort to many application areas, since these methods do not artificially create or exaggerate many types of structure in the data, as is possible with other models that encode strong prior assumptions.

\subsubsection{Principal Component Analysis}
\label{sec:pca}

Principal component analysis (PCA) was originally formulated by \cite{pearson1901} as a minimization of the sum of squared residual errors between projected data points and the original data $f_X(M) = || X - MM^{\top}X||_F^2 =   \sum_{i=1}^n || x_i - MM^{\top}x_i ||_2^2 $.  Modern treatments tend to favor the  equivalent ``maximizing variance" derivation (e.g., \citealp{bishopBook}), resulting in the objective $-\text{tr}(M^{\top}XX^{\top}M)$.  We write PCA in the formulation of Equation \ref{eq:1} as
\begin{equation}
\label{eq:pca}
\begin{aligned}
& \text{minimize}
& & || X - MM^{\top}X||_F^2  \\
& \text{subject to}
&& M \in \mathcal{O}^{d \times r}.
\end{aligned}
\end{equation}
Equation \ref{eq:pca} leads to the familiar SVD solution: after summarizing the data by its sample covariance matrix $\frac{1}{n}XX^{\top}$, the decomposition $XX^{\top} = Q\Lambda Q^{\top}$ produces an optimal point $M = Q_r$, where $Q_r$ denotes the columns of $Q$ associated with the largest $r$ eigenvalues of $XX^{\top}$ \cite[]{eckart1936,mirsky1960,golubBook}. 

There are many noteworthy extensions to PCA.  A first example is kernel PCA, which uses PCA on a feature space instead of the inputs themselves \cite[]{kpca}, and indeed some dimensionality reduction methods and their kernelized counterparts can be considered together as kernel regression problems \cite[]{delatorre2012}.  While quite important for all machine learning methods, we consider kernelized methods orthogonal to much of the presentation here, since using this kernel mapping is a question of representation of data, not of the dimensionality reduction algorithm itself.   

Second, there have been several probabilistic extensions to PCA, such as probabilistic PCA (PPCA; \citealp{tipping99, roweisSPCA}), extreme component analysis \cite[]{welling2003}, and minor component analysis \cite[]{agakov2002}.   These algorithms all share a common purpose (modeling covariance) and the same coordinate system for projection (the principal axes of the covariance ellipsoid), even though they differ in the particulars of the projection and which basis is chosen from that coordinate system.  We present PPCA as a separate algorithm below and leave the others as extensions of this core method.

Third, extensions have introduced outlier insensitivity via a different implicit noise model such as a Laplace observation model, leading to a few examples of robust PCA \cite[]{galpin1987, baccini1996, choulakian2006}.  An alternative approach to robust PCA is driven by the observation that a small number of highly corrupted observations can drastically influence standard PCA.  \cite{candes2011} takes this approach to robust PCA, considering the data as low-rank plus sparse noise.  Their results have particular theoretical and practical appeal and connect linear dimensionality reduction to the substantial nuclear-norm minimization literature.

Fourth, PCA has been made sparse in several contexts \cite[]{zou2006, daspremont2007, daspremont2008, journee2010}, where the typical PCA objective is augmented with a lasso-type $\ell_1$ penalty term, namely $f_X(M) = ||X - MM^\top X||_F^2 + \lambda ||M||_1$, with penalty term $\lambda$ and $||M||_1 = \sum_i\sum_j |M_{ij}|$.  This objective does not admit an eigenvalue approach, and as a result several specialized algorithms have been proposed.  Note however that this sparse objective is again simply a program over $\mathcal{O}^{d \times r}$ (albeit nondifferentiable).

Fifth, another class of popular extensions generalizes PCA to other exponential family distributions, beyond the implicit normal distribution of standard PCA \cite[]{collins2002,shakir2008}.  These methods, while important, result in nonlinear mappings of the data and thus fall outside the scope of Definition \ref{dfn:ldr}.   Additionally, there are other nonlinear extensions to PCA; Chapter 12.6 of \cite{icaBook} gives an overview.

\subsubsection{Multidimensional Scaling}
\label{sec:mds}

Multidimensional scaling (MDS; \citealp{torgerson1952, coxBook, borgBook}) is a class of methods and a large literature in its own right, but its connections to linear dimensionality reduction and PCA are so well-known that it warrants individual mention.  PCA minimizes low-dimensional reconstruction error, but another sensible objective is to maximize the scatter of the projection, under the rationale that doing so would yield the most informative projection (this choice is sometimes called classical scaling).  Defining our projected points $y_i  = M^\top x_i$ for some $M \in \mathcal{O}^{d \times r}$, MDS seeks to maximize pairwise distances $\sum_i \sum_j ||y_i - y_j||^2$.  

MDS leads to the seemingly novel optimization program (Equation \ref{eq:1}) over the scatter objective $f_X(M) = \sum_i \sum_j ||M^\top x_i - M^\top x_j||^2$, which can be expanded as
\begin{equation} 
\label{eq:mds2}
 f_X(M) \propto \text{tr}\left(M^\top XX^\top M\right) -  1^\top X^\top MM^\top X1 = \text{tr}\left(M^\top X \left(I - \frac{1}{n}11^\top\right)X^\top M \right),
\end{equation}
\noindent where we denote the vector of all ones as $1$.  Noting that $X$ has zero mean by definition and thus $X(I - \frac{1}{n}11^\top) = X$, we see classical MDS is precisely the `maximal variance' PCA objective $\text{tr}(M^\top XX^\top M)$.  

The equivalence of MDS and PCA in this special case is well-known \cite[]{coxBook, borgBook, mardiaBook, williams2002}, and indeed this particular example only scratches the surface of MDS, which is usually considered in much more general terms.  Specifically,  if we have available only pairwise distances $d_X(x_i,x_j)$, a more general MDS problem statement is to fit the low-dimensional data so as to preserve these pairwise distances as closely as possible in the least squares sense: minimizing $\sum_i \sum_j (d_X(x_i,x_j) - d_Y(y_i,y_j))^2$ is known as Kruskal-Shephard scaling, and the distance metrics can be arbitrary and different between the original and low-dimensional data.   First, it is worth noting that least squares is by no means the only appropriate stress function on the distances $d_X$ and $d_Y$; a Sammon mapping is another common choice (see for example \cite{hastieBook}, \S 14.8).  Second, MDS does not generally require the data itself, but only the pairwise dissimilarities $d_{ij} = d_X(x_i,x_j)$, which is often a useful property.  When the data is known, we see here that if we specify a low-dimensional orthogonal projection $Y = M^\top X$, then indeed this objective will result in the class of linear dimensionality reduction programs
\begin{equation}
\label{eq:mds0}
\begin{aligned}
& \text{minimize}
& & \sum_i \sum_j \left(d_X\left(x_i,x_j\right) - d_Y\left(M^\top x_i,M^\top x_j\right)\right)^2  \\
& \text{subject to}
&& M \in \mathcal{O}^{d \times r}.
\end{aligned}
\end{equation}
\noindent Special approaches exist to solve this program on a case-by-case basis \cite[]{coxBook, borgBook}.   However, by broadly considering Equation \ref{eq:mds0} as an optimization over orthogonal projections, we again see the motivation for a generic numerical solver for this class of problems, obviating objective-specific methods.  

Of course, the low-dimensional data $Y$ need not be a linear mapping of $X$ (indeed, in many cases the original points $X$ are not even available).    This more general form of MDS is used in a variety of nonlinear dimensionality reduction techniques, including prominently Isomap \cite[]{tenenbaum2000}, as discussed below in Section \ref{sec:outside}.

\subsubsection{Linear Discriminant Analysis}
\label{sec:lda}

Another natural problem-specific objective occurs when the data $X$ has associated class labels, of which Fisher's linear discriminant analysis (LDA; \citealp{fisher1936,rao1948}; modern references include \citealp{fukunagaBook, bishopBook}) is perhaps the most prominent example.  The purpose of LDA is to project the data in such a way that separation between classes is maximized.  To do so, LDA begins by partitioning the data covariance $XX^\top$ into covariance contributed within each of the $c$ classes ($\Sigma_W$) and covariance contributed between the classes ($\Sigma_B$) , such that $XX^\top = \Sigma_W + \Sigma_B$ for
\begin{equation}
\Sigma_W = \overset{n}{\sum_{i = 1}} (x_i - \mu_{c_i})(x_i - \mu_{c_i})^\top  ~~~~~~~~~~~~~ \Sigma_B = \overset{n}{\sum_{i = 1}} (\mu_{c_i} - \mu)(\mu_{c_i} - \mu)^\top,
\end{equation}
\noindent where $\mu$ is the global data mean (here $\mu=0$ by definition) and $\mu_{c_i}$ is the class mean associated with data point $x_i$. LDA seeks the projection that maximizes between-class variability $\text{tr}\left( M^\top \Sigma_B M \right)$ while minimizing within-class variability $\text{tr}\left( M^\top \Sigma_W M \right)$, leading to the optimization program
\begin{equation}
\label{eq:lda}
\begin{aligned}
& \text{maximize}
& & \frac{\text{tr}\left( M^\top \Sigma_B M \right)}{\text{tr}\left( M^\top \Sigma_W M \right)}  \\
& \text{subject to}
&& M \in \mathcal{O}^{d \times r}.
\end{aligned}
\end{equation}
\noindent This objective appears very much like a generalized Rayleigh quotient, and is so for $r=1$.  In this special case, $M \in \mathcal{O}^{d \times 1}$ can be found as the top eigenvector of $\Sigma_W^{-1} \Sigma_B^{}$, which can be seen by substituting $L = \Sigma_W^{{1}/{2}} M$ into Equation \ref{eq:lda} above.   This one-dimensional LDA projection is appropriate when there are $c=2$ classes. 

A common misconception is that LDA for higher dimensional projections $r>1$ can be solved with a greedy selection of the top $r$ eigenvectors of $\Sigma_W^{-1} \Sigma_B^{}$. However, this is certainly not the case, as the top $r$ eigenvectors of $\Sigma_W^{-1} \Sigma_B^{}$ will not in general be orthogonal.  The eigenvector solution solves the similar but not equivalent objective  $\text{tr}\left( \left(M^\top \Sigma_W M \right)^{-1}\left( M^\top \Sigma_B M \right)\right)$ over $M \in \reals^{d \times r}$; these two objectives and a few others are nicely discussed in Chapter 10 of \cite{fukunagaBook}.  While each of these choices has its merits, in the common case that one seeks a projection of the original data, the orthogonal $M$ produced by solving Equation \ref{eq:lda} is more appropriate.  Though rarely discussed, this misconception between the trace-of-quotient and the quotient-of-traces has been investigated in the literature \cite[]{yan2006, shen2007}. 

The commonality of this misconception adds additional motivation for this work, to survey and consolidate a fragmented literature.  Second, this misconception also points out the limitations of eigenvector approaches: even when considered the standard algorithm for a popular method, eigenvalue decompositions may in fact be an inappropriate choice.  Third, as Equation \ref{eq:lda} is a simple program over orthogonal projections, we see again the utility of a generic solver, an approach which should outperform traditional approaches (and indeed does, as Section \ref{sec:examples} will show).

In terms of extensions, we note a few key constraints of classical LDA:  each data point must be labeled with a class (no missing observations), each data point must be labeled with only one class (no mixed membership), and the class boundaries are modeled as linear.   As a first extension, one might have incomplete class labels; \cite{yu2006supervised} extends LDA (with a probabilistic PCA framework; see Section \ref{sec:ppca}) to the semi-supervised setting where not all points are labeled.  Second, data points may represent a mixture of multiple features, such that one wants to extract a projection where one feature is most discriminable.  \cite{brendel2011demixed} offers a possible solution by marginalizing covariances over each feature of interest.  Third,   \cite{mika1999} has extended LDA to the nonlinear domain via kernelization, which has also been well used.

\subsubsection{Canonical Correlations Analysis}
\label{sec:cca}

Canonical correlation analysis (CCA) is a problem of joint dimensionality reduction: given two data sets $X_a  \in \reals^{d_a \times n}$ and $X_b  \in \reals^{d_b \times n}$, find low-dimensional mappings $Y_a = P_aX_a$ and $Y_b=P_bX_b$ that maximize the correlation between $Y_a$ and $Y_b$, namely
{\footnotesize
\begin{equation}
\label{eq:corr}
 \rho\left( y_a , y_b \right)   = \frac{E\left( y_a ^\top y_b^{} \right)}{ \sqrt{E\left(y_a^\top y_a^{}\right)E\left(y_b^\top y_b^{}\right)}}  = \frac{ \mathrm{tr}\left( Y_a^{}Y_b^\top \right) }{  \sqrt{ \mathrm{tr}\left( Y_a^{}Y_a^\top \right)  \mathrm{tr}\left( Y_b^{}Y_b^\top \right) } }  
= \frac{ \mathrm{tr}\left(P_a^{} X_a^{} X_b^\top P_b^\top\right)}{\sqrt{\mathrm{tr}\left(P_a^{} X_a^{}X_a^\top P_a^\top \right)\mathrm{tr}\left(P_b^{} X_b^{} X_b^\top P_b^\top\right)}}.
\end{equation}
}
CCA was originally derived in \cite{hotelling1936}; more modern treatments include \cite{muirheadBook,  timm2002applied, hardoon2004, hardoon2009convergence}.   This method in its classical form, which we call  \emph{Traditional CCA}, seeks to maximize $\rho\left(y_a,y_b\right)$ under the constraint that all variables are uncorrelated and of unit variance: $\frac{1}{n}Y_a^{}Y_a^\top = I$,  $\frac{1}{n}Y_b^{}Y_b^\top = I$, and $Y_a^{}Y_b^\top = \Lambda$ for some diagonal matrix $\Lambda$.   As an optimization program over $P_a$ and $P_b$, Traditional CCA solves
\begin{equation}
\label{eq:cca_trad}
\begin{aligned}
& \text{maximize}
& &  \frac{ \mathrm{tr}\left(P_a^{} X_a^{} X_b^\top P_b^\top\right)}{\sqrt{\mathrm{tr}\left(P_a^{} X_a^{}X_a^\top P_a^\top\right)\mathrm{tr}\left(P_b^{} X_b^{} X_b^\top P_b^\top\right)}} \\
& \text{subject to} && \frac{1}{n} P_a^{}X_a^{}X_a^\top P_a^\top = I\\
& && \frac{1}{n} P_b^{}X_b^{}X_b^\top P_b^\top = I\\
& && P_a^{}X_a^{}X_b^\top P_b^\top = \Lambda.
\end{aligned}
\end{equation}
  Using the substitution $P_a = M_a^\top\left(X_a^{}X_a^\top\right)^{-1/2}$ for $M_a \in \mathcal{O}^{d_a \times r}$ (and similar for $P_b$), Traditional CCA reduces to the well known objective
\begin{equation}
\label{eq:cca_trad_simple}
\begin{aligned}
& \text{maximize}
& & \mathrm{ tr}\left( M_a^\top (X_a^{} X_a^\top)^{-1/2}X_a^{}X_b^\top (X_b^{} X_b^\top)^{-1/2} M_b^{} \right)  \\
& \text{subject to}
&& M_a \in \mathcal{O}^{d_a \times r}\\
& && M_b \in \mathcal{O}^{d_b \times r}.
\end{aligned}
\end{equation}
This objective is maximized when $M_a^\top$ is the top $r$ left singular vectors and $M_b^\top$ is the top $r$ right singular vectors of $(X_a^{} X_a^\top)^{-1/2}X_a^{}X_b^\top (X_b^{} X_b^\top)^{-1/2}$.  The linear transformations optimizing Equation \ref{eq:cca_trad} are then calculated as $P_a = M_a^\top (X_a^{} X_a^\top)^{-1/2}$, and similar for $P_b$.   This solution is provably optimal for any dimensionality $r$ under the imposed constraints \cite[]{muirheadBook}.

It is apparent by construction that $P_a$ and $P_b$ do not in general represent orthogonal projections (except when $X_a^{}X_a^\top = I$ and $X_b^{}X_b^\top = I$, respectively), and thus Traditional CCA is unsuitable for common settings (such as visualization of data in an orthogonal axis) where an orthogonal mapping is required.  In these cases, a common heuristic approach is to orthogonalize $P_a$ and $P_b$ to produce orthogonal mappings of the data $Y_a^{} = M_a^\top X_a^{}$ and $Y_b^{} = M_b^\top X_b^{}$.  This heuristic choice, however, produces suboptimal results for the original correlation objective of Equation \ref{eq:corr} for all dimensions $r > 1$ (the $r=1$ case is trivially an orthogonal projection), as the results will show. 

Our approach addresses a desire for orthogonal projections directly: with the optimization framework of Equation \ref{eq:1}, we can immediately write down a novel linear dimensionality reduction method that preserves Hotelling's original objective but is properly generalized to produce orthogonal projections.   We call this method \emph{Orthogonal CCA}, maximizing the correlation $\rho\left( y_a , y_b \right)$ objective directly over orthogonal matrices, namely
\begin{equation}
\label{eq:cca_orth}
\begin{aligned}
& \text{maximize}
& &\frac{ \mathrm{tr}\left(M_a^\top X_a^{} X_b^\top M_b^{}\right)}{\sqrt{\mathrm{tr}\left(M_a^\top X_a^{}X_a^\top M_a^{}\right)\mathrm{tr}\left(M_b^\top X_b^{} X_b^\top M_b^{}\right)}} \\
& \text{subject to}
&& M_a \in \mathcal{O}^{d_a \times r}\\
& && M_b \in \mathcal{O}^{d_b \times r}.
\end{aligned}
\end{equation}
The resulting low-dimensional mappings are then the orthogonal projections that we desire: $Y_a^{} = M_a^\top X_a^{}$ and  $Y_b^{} = M_b^\top X_b^{}$.   The optimization program of Equation \ref{eq:cca_orth} can not be solved with a known matrix decomposition, thus requiring a direct optimization approach.  More importantly, we point out the meaningful difference between Traditional CCA and Orthogonal CCA: Traditional CCA whitens each data set $X_a$ and $X_b$, and then orthogonally projects these whitened data into a common space such that correlation is maximized.  Orthogonal CCA on the other hand preserves the covariance of the original data $X_a$ and $X_b$, finding orthogonal projections where correlation is maximized without the initial whitening step.  It is unsurprising then that these two methods should return different mappings, even when the Traditional CCA result is orthogonalized post hoc.   Accordingly, CCA demonstrates the utility of considering linear dimensionality reduction in the framework of Equation \ref{eq:1}; methods can be directly written down for the objective and projection of interest, without having to shoehorn the problem into an eigenvector decomposition.

\subsubsection{Maximum Autocorrelation Factors}
\label{sec:maf}

There are a number of linear dimensionality reduction methods that seek to preserve temporally interesting structure in the projected data.
 A first simple example is maximum autocorrelation factors (MAF;  \citealp{switzer1984, larsen2002}).
Suppose the high-dimensional data $X \in \reals^{d \times n}$ has data points $x_t$ for $t\in \{1,...,n\}$, and that the index label $t$ defines an order in the data.
In such a setting, the structure of interest for the low-dimensional representation may have nothing to do with modeling data covariance (like PCA), but rather the appropriate description should include temporal structure.  

Assume that there is an underlying $r$-dimensional temporal signal that is smooth, and that the remaining $d-r$ dimensions are noise with little temporal correlation (less smooth).  MAF then seeks an orthogonal projection $P = M^\top$ for $M \in \mathcal{O}^{d \times r}$ so as to maximize correlation between adjacent points $y_t, y_{t+\delta}$, yielding the objective
\begin{equation}
\label{eq:maf}
f_X(M) = \rho( y_t , y_{t+\delta} ) =  \frac{E( y_t ^\top y_{t+\delta} )}{ \sqrt{E(y_t^2)E(y_{t+\delta}^2)}} = \frac{ E(x_t^\top MM^\top x_{t+\delta}^{}) }{E(x_t^\top MM^\top x_t^{})} = \frac{ \mathrm{tr}(M^\top \Sigma_\delta M)}{\mathrm{tr}(M^\top  \Sigma M)}, 
\end{equation}
\noindent  where $\Sigma$ is the empirical covariance of the data $E(x_t^{} x_t^\top) = \frac{1}{n}XX^\top$ and $\Sigma_\delta$ is the symmetrized empirical cross-covariance of the data evaluated at a one-step time lag  $\Sigma_\delta = \frac{1}{2} \left( E(x_{t+\delta}^{} x_t^\top ) + E(x_{t}^{} x_{t+\delta}^\top ) \right)$.  
This objective results in the linear dimensionality program
\begin{equation}
\label{eq:maf2}
\begin{aligned}
& \text{maximize}
& &  \frac{ \mathrm{tr}(M^\top \Sigma_\delta M)}{\mathrm{tr}(M^\top  \Sigma M)}\\
& \text{subject to}
&& M \in \mathcal{O}^{d \times r}.
\end{aligned}
\end{equation}
\noindent Note again the appearance of the quotient-of-traces objective (as in LDA and CCA).  Indeed, the same heuristic (solving the trace-of-quotient problem) is typically applied to MAF, which results in the standard choice of the top eigenvectors of  $\Sigma^{-1} \Sigma_\delta$ as the solution to Equation \ref{eq:maf2}.  Though correct in the $r=1$ case, this misconception is incorrect for precisely the same reasons as above with LDA, and its use results in the same pitfalls.   Directly solving the manifold optimization of Equation \ref{eq:1} presents a more straightforward option.

MAF can be seen as a method balancing the desire for cross-covariance ($\Sigma_\delta$) of the data without overcounting data that has high power (the denominator containing $\Sigma$).  Indeed, such methods have been invented with slight variations in various application areas (e.g., \citealp{cunningham2014dimensionality}).
For example, one might simply ask to maximize the cross-covariance $E( y_t^\top y_{t+\delta} )$ rather than the correlation itself.  Doing so results in a simpler problem than Equation \ref{eq:maf2}: maximize $\mathrm{tr}( M^\top \Sigma_\delta M )$ for $M \in \mathcal{O}^{d \times r}$. 
In this case the eigenvector solution is optimal.  Second, we may want to maximize (or minimize, as in \cite{turner2007}) the squared distance between projected points; the objective then becomes $E( || y_{t+\delta} - y_t ||^2 )$, which through a similar set of steps produces the similar eigenvalue problem $\mathrm{tr}\left( M^\top (\Sigma - \Sigma_\delta )  M \right)$ for $M \in \mathcal{O}^{d \times r}$.
 This last choice is a discrete time analog of a more popular method---slow feature analysis---which we discuss in the next section.
Third, one might want to specify a particular form of temporal structure in terms of a dynamics objective $f_X(M)$, and seek linear projections containing that structure.   The advantage of such an approach is that one can specify a range of dynamical structures well beyond the statistics captured by an autocorrelation matrix.  A recent simple example is \cite{churchland2012}, who sought a linear subspace of the data where linear dynamics were preserved, namely an $M$ minimizing $f_X(M) = ||\dot{X} - MDM^\top X||^2_F$ for some dynamics matrix $D\in \reals^{r \times r}$.  This objective is but one simple choice of dynamical structure; given the canonical autonomous system $\dot{y} = g(y) + \epsilon$, one might similarly optimize $f_X(M) = || M^\top \dot{X} - g(M^\top X) ||_F^2$.  Optimizing such a program finds the projection of the data that optimally expresses that dynamical feature of interest, without danger of artificially creating that structure based on a strong prior model (as is possible in state-space models like the Kalman filter).

\subsubsection{Slow Feature Analysis}
\label{sec:sfa}

Similar in spirit to MAF, slow feature analysis (SFA; \citealp{wiskott2002, wiskott2003}) is a linear dimensionality reduction technique developed to seek invariant representations in object recognition problems.  SFA assumes that measured data, such as pixels in a movie, can have rapidly changing values over time, whereas the identity, pose, or position of the underlying object should move much more slowly.  Thus, recovering a slowly moving projection of data may produce a meaningful representation of the true object of interest.   Accordingly, assuming access to derivatives
$\dot{X} = [\dot{x}_1 , ... , \dot{x}_n]$, SFA minimizes the trace of the covariance of the projection $\mathrm{tr}(\dot{Y}\dot{Y}^\top) = \mathrm{tr}(M^\top \dot{X}\dot{X}^\top M)$.  This objective is PCA on the derivative data:
\begin{equation}
\label{eq:sfa}
\begin{aligned}
& \text{minimize}
& &  \mathrm{tr}\left(M^\top \dot{X}\dot{X}^\top M\right)\\
& \text{subject to}
&& M \in \mathcal{O}^{d \times r}.
\end{aligned}
\end{equation}
Linear SFA is the most straightforward case of the class of SFA methods.  Several additional choices are typical in SFA implementations, including: \emph{(i)} data points  $x_t \in X$ are usually expanded nonlinearly via some feature mapping $h : \reals^{d} \rightarrow \reals^{p}$ for some $p > d$ (a typical choice is all monomials of degree one and two to capture linear and quadratic effects); and \emph{(ii)} data are whitened to prevent the creation of structure due to the mapping $h(\cdot)$ alone before the application of the PCA-like program in Equation \ref{eq:sfa}.   A logical extension of this nonlinear feature space mapping is to consider a reproducing kernel Hilbert space mapping, as has indeed been done \cite[]{bray2002}.  \cite{turner2007} established the connections between SFA and linear dynamical systems, giving a probabilistic interpretation of SFA that also makes different and interesting connections of this method to PCA and its probabilistic counterpart (Section \ref{sec:ppca}).

\subsubsection{Sufficient Dimensionality Reduction}
\label{sec:sdr}

Consider a supervised learning problem with high dimensional covariates $X \in \reals^{d \times n}$ and responses $Z \in \reals^{\ell \times n}$.  The concept behind sufficient dimensionality reduction is to find an orthogonal projection of the data $Y = M^\top X \in \reals^{r \times n}$ such that the reduced-dimension points $Y$ capture all statistical dependency between $X$ and $Z$.  Thus, sufficient dimensionality reduction (SDR) is a problem of feature selection that seeks an $M \in \mathcal{O}^{d \times r}$ which makes covariates and responses conditionally independent:
\begin{equation}
\label{eq:ci}
p_{Z|X}(z|x) = p_{Z|M^\top X}(z|M^\top x) ~~\iff ~~  Z \perp\!\!\!\perp X | M^\top X.
\end{equation}
SDR is in fact a class of methods, as there are a number of ways one might derive an objective for such a conditional independence relationship.  Particularly popular in machine learning is the use of kernel mappings to characterize the conditional independence relationship of Equation \ref{eq:ci} \cite[]{fukumizu2004dimensionality, fukumizu2009kernel, nilsson2007regression}.  The essential idea in these works is to map covariates $X$ and responses $Z$ into reproducing kernel Hilbert spaces, where it has been shown that, for universal kernels, cross-covariance operators can be used to determine conditional independence of $X$ and $Z$ \cite[]{fukumizu2004dimensionality, gretton2012kernel, gretton2005measuring}.   Such an approach induces the cost function on the projection
\begin{equation}
\label{eq:sdr0}
f_{X}(M) = J( Z , M^\top X ) := \mathrm{tr}\left( \bar{K}_Z \left( \bar{K}_{M^\top X} + n\epsilon I \right)^{-1} \right),
\end{equation}
\noindent where $\bar{K}_Z = \left(I - \frac{1}{n} 1 1^\top \right) K_Z \left(I - \frac{1}{n} 1 1^\top \right)$ is the centered Gram matrix $K_Z = \{k(z_i,z_j)\}_{ij}$ (and similar for $\bar{K}_{M^\top X}$).   Critically, this cost function is provably larger than $J( Z , X )$, with equality if and only if the desired conditional independence of Equation \ref{eq:ci} holds.  Thus, we have the following linear dimensionality reduction program: 
\begin{equation}
\label{eq:sdr}
\begin{aligned}
& \text{minimize}
& &  \mathrm{tr}\left( \bar{K}_Z \left( \bar{K}_{M^\top X} + n\epsilon I \right)^{-1} \right)\\
& \text{subject to}
&& M \in \mathcal{O}^{d \times r}.
\end{aligned}
\end{equation}
SDR has been extended to the unsupervised case \cite[]{wang2010unsupervised} and has been implemented with other objectives such as the Hilbert-Schmidt independence criterion \cite[]{gretton2005measuring}.  An important review of non-kernel SDR techniques is \cite{adragni2009}, in addition to earlier work \cite[]{li1991sliced}.

\subsubsection{Locality Preserving Projections}
\label{sec:lpp}

All methods considered thus far stipulate objectives based on global loss functions, which can be sensitive to outliers and can be significantly distorted by nonlinear structure in the data.  A popular alternative throughout machine learning is to consider local neighborhood structure.  In the case of dimensionality reduction, considering locality often amounts to constructing a neighborhood graph of the training data, and using that graph to define the loss function.  Numerous \emph{nonlinear} methods have been proposed along these lines (see Section \ref{sec:outside}), and this development  has led to a few important \emph{linear} methods that consider local structure. 
First, locality preserving projections (LPP) \cite[]{niyogi2004locality} is a direct linear interpretation of Laplacian Eigenmaps \cite[]{belkin2003}.  LPP begins by defining a graph with each data point $x_i \in \reals^d$ as a vertex, connecting $x_i$ and $x_j$ with the edge $\delta_{i,j}$ if these points are in the same $\epsilon$ neighborhood (that is, $||x_i -x_j|| < \epsilon$).  A kernel (typically the squared exponential kernel) is then used to weight the existing edges.   The cost of the reconstruction $y_i = P x_i$ is then
\begin{equation}
\overset{n}{\sum_{i=1}}\overset{n}{\sum_{j=1}}  || P x_i - P x_j ||^2_2 W_{ij} ~~, ~~~\mathrm{where} ~~~ W_{ij} = \delta_{i,j} \exp\left\{ -\frac{1}{\tau} ||x_i - x_j ||^2_2 \right\}.
\end{equation}
Through a few standard steps (see for example \citealp{belkin2003,niyogi2004locality}), this objective results in the linear dimensionality objective
\begin{equation}
\label{eq:lpp}
\begin{aligned}
& \text{minimize}
& &  \mathrm{tr}\left(P X L X^\top P^\top \right) \\
& \text{subject to}
&& P X D X^\top P^\top = I, \\
\end{aligned}
\end{equation}
\noindent where the matrix $D$ is diagonal with the column sums of $W$, namely $D_{ii} = \sum_j W_{ij}$, and $L = D-W$ is the Laplacian matrix. Note that this constraint set is sometimes called a \emph{flag} matrix manifold.   As in Traditional CCA (Section \ref{sec:cca}), LPP can be solved to produce the matrix $P$ with columns equal to the generalized eigenvectors $v_i$ satisfying $XLX^\top v_i = \lambda_i XDX^\top v_i$, by implicitly solving an orthogonally constrained optimization over $M = (XDX^\top)^{1/2} P^\top \in \mathcal{O}^{d \times r}$: 
\begin{equation}
\label{eq:sdr1}
\begin{aligned}
& \text{minimize}
& &  \mathrm{tr}\left( M^\top (XDX^\top)^{-\top/2} XLX^\top (XDX^\top)^{-1/2} M \right)\\
& \text{subject to}
&& M \in \mathcal{O}^{d \times r}.
\end{aligned}
\end{equation}
Note again that the resulting linear mapping $Y = PX = M^\top (XDX^\top)^{-\top/2} X$ is not an orthogonal projection.    

Related to LPP, neighborhood preserving embedding (NPE) \cite[]{he2005neighborhood} has a largely parallel motivation.  NPE is a linear analogue to locally linear embedding \cite[]{roweis2000} that produces a different linear approximation to the Laplace Beltrami operator, resulting in the objective
\begin{equation}
\label{eq:lpp2}
\begin{aligned}
& \text{minimize}
& &  \mathrm{tr}\left( M^\top (XX^\top)^{-\top/2} X(I-W)^\top(I-W) X^\top (XX^\top)^{-1/2} M \right)\\
& \text{subject to}
&& M \in \mathcal{O}^{d \times r},
\end{aligned}
\end{equation}
\noindent where the matrix $W$ is the same in LPP.  We group these methods together due to their similarity in motivation and resulting objective.

\subsection{Linear Dimensionality Reduction with Unconstrained Objectives}
\label{sec:nonorth}

All methods reviewed so far involve  orthogonal mappings, but several methods simplify further to an unconstrained optimization over matrices $M \in \reals^{d \times r}$.  We describe those linear dimensionality reduction methods here.

\subsubsection{Undercomplete Independent Component Analysis}
\label{sec:uica}

Independent Component Analysis (ICA; \citealp{icaBook}) is a massively popular class of methods that is often considered alongside PCA and other simple linear transformations.  ICA specifies the usual data $X \in \reals^{d \times n}$ as a mixture of unknown and independent sources $Y \in \reals^{r \times n}$.  Note the critical difference between the independence requirement and the uncorrelatedness of PCA and other methods: for each source data point $y = [y^1,...,y^r]^\top \in \reals^r$ (one column of $Y$), independence implies $p(y) \approx \prod_{j=1}^r p\left(y^j\right)$, where the $p\left(y^j\right)$ are the univariate marginals of the low dimensional data (sources).
  
ICA finds the demixing matrix $P$ such that we recover the independent sources as $Y = P X$.  The vast majority of implementations and presentations of ICA deal with the dimension preserving case of $r=d$, and indeed most widely used algorithms require this parity.  In this case, ICA is not a dimensionality reduction method.    

Our case of interest for dimensionality reduction is the `undercomplete' case where $r < d$, in which case $Y = P X$ is a linear dimensionality reduction method according to Definition \ref{dfn:ldr}.  Interestingly, the most common approach to undercomplete ICA is to preprocess the mixed data $X$ with PCA (e.g., \citealp{joho2000}), reducing the data to $r$ dimensions, and running a standard square ICA algorithm.   That said, there are a number of principled approaches to undercomplete ICA, including \cite[]{stone1998, zhang1999, amari1999, de2002texture, welling2004}.    All of these models necessarily involve a probabilistic model, required by the independence of the sources.    As an implementation detail, note that observations $X$ are whitened as a preprocessing step.  

With this model, authors have maximized the log-likelihood of a generative model \cite[]{de2002texture} or minimized the mutual information between the sources \cite[]{stone1998, amari1999, zhang1999}, each of which requires an approximation technique. 
   \cite{welling2004} describes an exact algorithm for maximizing the log-likelihood of a product of experts objective
\begin{equation}
\label{eq:uica}
f_X(M) ~~ = ~~ \frac{1}{n} \overset{n}{\sum_{i=1}} \log p(x_i) ~~ \propto ~~ \frac{1}{2} \log | M^\top M | + \frac{1}{n} \overset{n}{\sum_{i=1}}  \overset{r}{\sum_{k=1}} \log p_\theta\left( m_k^\top x_i^{} \right),
\end{equation}   
where $m_k$ are the (unconstrained) columns of $M$, and $p_\theta(\cdot)$ is a likelihood distribution (an ``expert") parameterized by some $\theta_k$.  Thus this undercomplete ICA, as an optimization program like Equation \ref{eq:1}, is a simple unconstrained maximization of $f_X(M)$ over $M \in \reals^{d\times r}$.   

Extensions of ICA are numerous.  Insomuch as undercomplete ICA is a special case of ICA, many of these extensions will also be applicable in the undercomplete case; see the reference \cite{icaBook}.

\subsubsection{Probabilistic PCA}
\label{sec:ppca}

 One often-noted shortcoming of PCA is that it partitions data into orthogonal signal (the $r$-dimensional projected subspace) and noise (the $(d-r)$-dimensional nullspace of $M^\top$).   Furthermore, PCA lacks an explicit generative model.  Probabilistic PCA  (PPCA; \citealp{tipping99, roweisSPCA, theobald75}) adds a prior to PCA to address both these potential concerns, treating the high-dimensional data to be a linear mapping of the low-dimensional data (plus noise).  If we stipulate some latent independent, identically distributed $r$-dimensional data $y_i \sim \mathcal{N}(0,I_r)$ for $i \in \{1,...,n\}$, and we presume that the high-dimensional data is a noisy linear mapping of that low-dimensional data $x_i | y_i \sim \mathcal{N}(My_i , \sigma_\epsilon^2 I )$ for some given or estimated noise parameter $\sigma_\epsilon^2$.  This model yields a natural objective with the total (negative log) data likelihood, namely
\begin{equation}
\label{eq:ppca1}
f_X(M) ~=~ -\log p(X | M) ~~\propto~~  \log | MM^\top +  \sigma_\epsilon^2 I | + \mathrm{trace} \left( (MM^\top +  \sigma_\epsilon^2 I )^{-1}  XX^\top \right).
\end{equation}
\noindent 
Mapping this onto our dimensionality reduction program, we want to minimize the negative log likelihood $f_X(M)$ over an arbitrary matrix $M \in \reals^{d \times r}$.
Appendix A of \cite{tipping99} shows that this objective can be minimized in closed form as $M = U_r(S_r - \sigma_\epsilon^2 I)^{\frac{1}{2}}$ where $\frac{1}{n}XX^\top = USU^\top$ is the singular value decomposition of the empirical covariance, and $U_r$ denotes the first $r$ columns of $U$ (ordered by the singular values).   \cite{tipping99} also show that the noise parameter $\sigma_\epsilon^2$ can be solved in closed form, resulting in a closed-form maximum likelihood solution to the parameters of PPCA.  This closed-form obviates a more conventional expectation-maximization (EM) approach \cite[]{dempster77}, though in practice EM is still used with the Sherman-Morrison-Woodbury matrix inversion lemma for computational advantage when $d\gg r$.
Under this statistical model, the low-dimensional mapping of the observed data is the mean of the posterior $p(Y | X)$, which also corresponds to the MAP estimator:  $Y = M^\top (MM^\top + \sigma_\epsilon^2 I )^{-1}X$, which again fits the form of linear dimensionality reduction $Y = PX$.

As with PCA, there are a number of noteworthy extensions to PPCA.   \cite{ulfarsson2008} add an $\ell_2$ regularization term to the PPCA objective.  This regularization can be viewed as placing a Gaussian shrinkage prior $p(M)$ on the entries of $M$, though the authors termed this choice more as a penalty term to drive a sparse solution.  A different choice of regularization is found in 
``Directed" PCA \cite[]{vanroy2011}, where a trace penalty on the inverse covariance matrix is added.    
Finally, more generally, several of the extensions noted in Section \ref{sec:pca} are also applicable to the probabilistic version.

\subsubsection{Factor Analysis}
\label{sec:fa}

Factor analysis (FA; \citealp{spearman1904}) has become one of the most widely used statistical methods, in particular in psychology and  behavioral sciences.   FA is a more general case of a PPCA model: the observation noise is fit per observation rather than across all observations, resulting in the following conditional data likelihood:  $x_i | y_i \sim \mathcal{N}(My_i , D )$ for a diagonal matrix $D$,  where the matrix $M$ is typically termed factor loadings.  This choice can be viewed as a means to add scale invariance to each measurement, at the cost of losing rotational invariance across observations.  Following the same steps as in PPCA, we arrive at the linear dimensionality reduction program
\begin{equation}
\label{eq:fa}
\begin{aligned}
& \text{minimize}
& &  \log | MM^\top + D | + \mathrm{trace}\left( (MM^\top +  D )^{-1}  XX^\top \right),
\end{aligned}
\end{equation}
\noindent which results in a similar linear dimensionality reduction mapping $Y=PX$ for $P =  M^\top (MM^\top + D )^{-1}$.   Unlike PPCA, FA has no known closed-form solution, and thus an expectation-maximization algorithm \cite[]{dempster77} or direct gradient method is typically used to find a (local) optimum of the log likelihood. Extensions similar to those for PPCA have been developed for FA (see for example \cite{vanroy2011}).

\subsubsection{Linear Regression}
\label{sec:lr}

Linear regression is one of the most basic and popular tools for statistical modeling.   Though not typically considered a linear dimensionality reduction method, this technique maps $d$-dimensional data onto an $r$-dimensional hyperplane defined by the number of independent variables.  Considering $d$-dimensional data $X$ as being partitioned into inputs and outputs $X = [X_{in};X_{out}]$ for inputs $X_{in} \in \reals^{r \times n}$ and outputs $X_{out} \in \reals^{(d-r) \times n}$, linear regression fits $X_{out} \approx MX_{in}$ for some parameters $M \in \reals^{(d-r)\times r}$.  The standard choice for fitting such a model is to minimize a simple sum-of-squared-errors objective $f_X(M) = || X_{out} - MX_{in} ||^2_F$, which leads to the least squares solution $M = X_{out}^{}X_{in}^\top (X_{in}^{}X_{in}^\top)^{-1}$.   In the form of Equation \ref{eq:1}, linear regression is
\begin{equation}
\label{eq:lr}
\text{minimize}
~~~ || X_{out} - MX_{in}||_F^2
\end{equation}
This model produces a regressed data set $\hat{X} = [X_{in}; MX_{in}] = [I;M]X_{in}$.   Note that $[I;M]$ has rank $r$ (the data lie on a $r$-dimensional subspace) and thus Definition \ref{dfn:ldr} applies.   To find the dimensionality reduction mapping $P$, we simply take the SVD $[I;M] = USV^\top$ and set $P = [SV^\top ~0 ]$ where $0$ is the $(d-r) \times (d-r)$ matrix of zeroes.   The low dimensional mapping of the original data $X$ then takes the standard form $Y = PX$.
 Chapter 3 of \cite{hastieBook} gives a thorough introduction to linear regression and points out (Equation 3.46) that the least squares solution can be viewed as mapping the output $X_{out}$ in a projected basis.  \cite{adragni2009} point out linear regression as a dimensionality reduction method in passing while considering the case of sufficient dimensionality reduction (see SDR, Section \ref{sec:sdr}, for more detail).  

An important extension to linear regression is regularization for bias-variance tradeoff, runtime performance, or interpretability of results.   The two most popular include adding an $\ell_2$ (ridge or Tikhonov regression) or an $\ell_1$ penalty (lasso), resulting in the objective
\begin{equation}
\label{eq:rlr}
\text{minimize}
~~~ || X_{out} - MX_{in}||_F^2  + \lambda ||M||_p
\end{equation}
for some penalty $\lambda$.  While the $\ell_2$ case can be solved in closed form as an augmented least squares, the $\ell_1$ case requires a quadratic program \cite[]{tibshirani1996}; though the simple quadratic program formulation scales poorly \cite[]{boyd2011distributed,bach2011convex}.  Regardless, both methods produce an analogous form as in standard linear regression, resulting in a linear dimensionality reduction $Y = PX$ for $P = [SV^\top ~0 ]$ as above.  

Another important extension, particularly given the present subject of dimensionality reduction, is principal components regression and partial least squares \cite[]{hastieBook}.   Principal components regression uses PCA to preprocess the input variables $X_{in} \in \reals^{r \times n}$  down to a reduced $\tilde{X}_{in} \in \reals^{\tilde{r} \times n}$, where $\tilde{r}$ is chosen by computational constraints, cross-validation, or similar.  Standard linear regression is then run on the resulting components.  This two-stage method (first PCA, then regression) can produce deeply suboptimal results, a shortcoming which to some extent is answered by partial least squares.  Partial least squares is another classical method that trades off covariance of $X_{in}$ (as in the PCA step of principal components regression) and predictive power (as in linear regression).  Indeed, partial least squares has been shown to be a compromise between linear regression and principal components regression, using the framework of continuum regression \cite[]{stone1990continuum}.   Even still, the partial least squares objective is heuristic and is carried out on $r$ dimensions in a greedy fashion.  \cite{bakir2004multivariate} approached the rank-$r$ linear regression problem directly, writing the objective in the form of Equation \ref{eq:1} as
\begin{equation}
\label{eq:ranklr}
\begin{aligned}
& \text{minimize}
& &  || X_{out}^{} - M_{out}^{} SM_{in}^\top X_{in}^{}||_F^2 \\
& \text{subject to}
&& M_{out} \in \mathcal{O}^{d_{out} \times r}\\
& && M_{in} \in \mathcal{O}^{d_{in} \times r},
\end{aligned}
\end{equation}
\noindent where $S$ is a nonnegative diagonal matrix, and the optimization program is over the variables $\left\{M_{in}, M_{out} , S \right\}$.  This method can again be solved as an example of Equation \ref{eq:1}.

\subsubsection{Distance Metric Learning}
\label{sec:dml}

Distance metric learning (DML) is an important class of machine learning methods that is typically motivated by the desire to improve a classification method.  Numerous algorithms---canonical examples include $k$-nearest neighbors and support vector machines---calculate distances between training points, and the performance of these algorithms can be improved substantially by a judicious choice of distance metric between these points.  Many objectives have been proposed to learn these distance metrics; a seminal work is \cite{xing2002distance}, and thorough surveys of this literature include \cite{kulis2012metric, yang2006distance, yang2007overview}.

In the linear case, to generalize beyond Euclidean distance, distance metric learning seeks a Mahalanobis distance $d_M(x_i,x_j) = ||M^\top x_i - M^\top x_j||_2 = ||x_i - x_j||_{MM^\top}$ that improves some objective on training data.  When $M \in \reals^{d \times d}$ is full rank, this approach is not a dimensionality reduction.  However, as is often noted in that literature, a lower rank $M \in \reals^{d \times r}$ for $r<d$ implies a linear mapping of the data to some reduced space where classification (or another objective) is hopefully improved, thus implicitly defining a linear dimensionality reduction method.    

Numerous methods have been introduced in the DML literature.   Here for clarity we survey one representative method in depth and incorporate other popular approaches from this literature  thereafter.  Large margin nearest neighbors (LMNN; \citealp{weinberger2005distance, torresani2006large, weinberger2009distance}) assumes labeled data: $(x_i,z_i)$, such that $z_i \in \{1,...,C\}$ for the $C$ data classes.  LMNN typically begins by identifying a target neighbor set $\eta(i)$ for each data point $x_i$, which, in the absence of side information, is simply the $k$ nearest neighbors belonging to the same class $z_i$ as point $x_i$.  The key intuition behind LMNN is that a distance metric $d_M(x_i,x_j)$ is desired such that target neighbors are pulled closer together than any points belonging to a different class, ideally with a large margin.  Accordingly, LMNN optimizes the objective
\begin{equation}
f_X(M) = \overset{n}{\sum_{i=1}} \overset{}{\sum_{j \in \eta(i)}}\left( d_M(x_i,x_j)^2 + \lambda  \overset{n}{\sum_{\ell = 1}} \mathds{1}(z_i \neq z_\ell)\left[ 1 + d_M(x_i,x_j)^2 - d_M(x_i,x_\ell)^2 \right]_+\right),
\end{equation} 
\noindent  where $\mathds{1}(\cdot)$ is the indicator function for the class labels $z_i,z_\ell$, and $[\cdot]_+$ is the hinge loss.  Intuitively, the first term of the right hand side pulls target neighbors closer together, while the second term penalizes (with weight $\lambda$) any points $x_\ell$ that are closer to $x_i$ than its target neighbors $x_j$ (plus some margin), and have a different label ($z_i \neq z_\ell$). 

As a dimensionality reduction technique, this objective is readily optimized over $M \in \reals^{d\times r}$, to produce a low dimensional mapping of the data $Y = M^\top X$.  Beyond LMNN, other prominent methods explore slightly different objectives with similar motivations.  Examples include relevant component analysis for DML \cite[]{bar2003learning}, neighborhood component analysis \cite[]{goldberger2004neighbourhood}, collapsing classes \cite[]{globerson2005metric}, discriminative component analysis \cite[]{peltonen2007fast}, latent coincidence analysis \cite[]{der2012latent}, and an online, large-scale method \cite[]{chechik2009online}.  Many of these works also offer kerneled extensions for nonlinear DML.

\subsection{Scope Limitations}
\label{sec:outside}

Definition \ref{dfn:ldr} limits  our scope and excludes a number of algorithms that could be considered dimensionality reduction methods.  Here we consider four prominent cases that fall outside the definition of linear dimensionality reduction.  

\subsubsection{Nonlinear Manifold Methods}

The most obvious methods to exclude from linear dimensionality reduction are nonlinear manifold methods, the most popular of which include Local Linear Embedding \cite[]{roweis2000}, Isomap \cite[]{tenenbaum2000}, Laplacian eigenmaps \cite[]{belkin2003}, maximum variance unfolding \cite[]{weinberger2006}, t-distributed stochastic neighbor embedding \cite[]{van2008visualizing}, and diffusion maps \cite[]{coifman2006}.  These methods seek a nonlinear manifold by using local neighborhoods, geodesic distances, or other graph theoretic considerations.  Thus, while these methods are an important contribution to dimensionality reduction, they do not produce low-dimensional data as $Y = PX$ for any $P$.   It is worth noting that some of these problems, such as Laplacian eigenmaps, do involve a generalized eigenvector problem in their derivation, though typically those eigenproblems are the direct solution to a stated objective and not the heuristic that is more often seen in the linear setting (and that motivates the use of direct optimization).    
A concise introduction to nonlinear manifold methods is given in \cite{zhao2007}, an extensive comparative review is \cite{van2009}, and a  probabilistic perspective on many spectral methods is given in \cite{lawrence2012unifying}.  

\subsubsection{Nonparametric Methods}

One might also consider classical methods from linear systems theory, like Kalman filtering or smoothing \cite[]{kalman60}, as linear dimensionality reduction methods.  Even more generally, nonparametric methods like Gaussian Processes \cite[]{rasmussenBook} also bear some similarity.  The key distinction with these algorithms is that our definition of linear dimensionality is parametric: $P \in \reals^{r \times d}$ is a fixed mapping and does not change across the data set or some other index.  Certainly any nonparametric method violates this restriction, as by definition the transformation mapping must grow with the number of data points.
 In the Kalman filter, for example, the mapping (which is indeed linear) between each point $x_i$ and its low-dimensional projection $y_i$ changes with each data point (based on all previous data), so in fact this method is also a nonparametric mapping that grows with the number of data points $n$.  This same argument applies to most state-space models and subspace identification methods, including the linear quadratic regulator, linear quadratic Gaussian control, and similar.   Hence these other classic methods also fall outside the scope of linear dimensionality reduction.

\subsubsection{Matrix Factorization Problems}
\label{sec:mf}

A few methods discussed in this work have featured matrix factorizations, and indeed there are many other methods that involve such a decomposition in areas like indexing and collaborative filtering.  This general class certainly bears similarity to dimensionality reduction, in that it uses a lower dimensional set of factors to reconstruct noisy or missing high-dimensional data (for example, classical latent semantic indexing is entirely equivalent to PCA; \citealp{deerwester1990}).   A common factorization objective is to find $ H \in \reals^{d \times r} $ and $Y \in \reals^{r \times n}$ such that the product $HY$ reasonably approximates $X$ according to some criteria.  The critical difference between these methods and linear dimensionality reduction is that these methods do not in general yield a sensible linear mapping $Y= PX$, but rather the inverse mapping from low-dimension to high-dimension.  While this may seem a trivial and invertible distinction, it is not: specifics of the method often imply that the inverse mapping is nonlinear or ill-defined.   To demonstrate why this general class of problem falls outside the scope of linear dimensionality reduction, we detail two popular examples: nonnegative matrix factorization and matrix factorization as used in collaborative filtering.

Nonnegative matrix factorization (NMF; \citealp{lee1999}; sometimes called multinomial PCA; \citealp{buntine02}), 
solves the objective $f_X(H,Y) = || X - HY ||$ for a nonnegative linear basis  $H \in \reals^{d \times r}_+$ and a nonnegative low-dimensional mapping $Y \in \reals^{r \times n}_+$.  The critical difference with our construction is that NMF is not linear: there is no $P$ such that $Y=PX$ for all points $x_i$.  If we are given $H$ and a test point $x_i$,  we must do the nonlinear solve $y_i = \mathrm{argmin}_{y \geq 0} || x_i - Hy ||_2$.  A simple counterexample is to take an existing point $x_j$ and its nonnegative projection $y_j$ (which we assume is not zero).  If we then test on $-x_j$, certainly we can not get $-y_j$ as a valid nonnegative projection.

A second example is the broad class of matrix factorization problems as used in collaborative filtering, which includes weighted low-rank approximations \cite[]{srebro2003}, maximum margin matrix factorization \cite[]{srebro2004, rennie2005}, probabilistic matrix factorization \cite[]{mnih2007}, and more.  As above, collaborative filtering algorithms approximate data $X$ with a low-dimensional factor model $HY$.  However, the goal of collaborative filtering is to fill in the missing entries of $X$ (e.g., to make movie or product recommendations), and indeed the data matrix $X$ is usually missing the vast majority of its entries.   Thus, not only is there no explicit dimensionality reduction $Y=PX$, but that operation is not even well defined for missing data.  

More broadly, there has been a longstanding literature in linear algebra of low rank approximations and matrix nearness problems, often called Procrustes problems  \cite[]{higham1989, li2011, ruhe1987, schonemann1966}.   These optimization programs have the objective $f_X(M) = || X - M ||$ for some norm (often a unitarily invariant norm, most commonly the Frobenius norm) and some constrained, low-rank matrix $M$.  PCA would be an example, considering $X$ as the data (or the covariance) and $M$ as the $r$-rank approximation thereof.  While a few linear dimensionality reduction methods can be written as Procrustes problems, not all can, and thus nothing general can be claimed about the connection between Procrustes problems and the scope of this work. 

\subsection{Summary of the Framework}
Table \ref{tbl:results} offers a consolidated summary of these methods.  Considering linear dimensionality reduction through the lens of a constrained matrix optimization enables a few key insights.  First, as is the primary purpose of this paper, this framework surveys and consolidates the space of linear dimensionality reduction methods.  It clarifies that linear dimensionality reduction goes well beyond PCA and can require much more than simply eigenvalue decompositions, and also that many of these methods bear significant resemblance to each other in spirit and in detail.   Second, this consolidated view suggests that, since optimization programs over well-understood matrix manifolds address a significant subclass of these methods, an objective-agnostic solver over matrix manifolds may provide a useful generic solver for linear dimensionality reduction techniques.  
%
\begin{table*}[]
\scriptsize
\centering
\begin{tabular}[b]{p{1.9cm}ccc}\\[-7pt]
\hline\hline\\[-7pt]
Method & Objective $f_X(M)$ & Manifold $\mathcal{M}$ & Mapping $Y=PX$  \\[2pt]
\hline\\[-8pt]

\\[-11pt] \\
PCA (\S \ref{sec:pca}) & $|| X - MM^{\top}X||_F^2$ & $\mathcal{O}^{d \times r}$ & $M^\top X$ \\

\\[-9pt] \\
MDS (\S \ref{sec:mds}) & $  \sum_{i,j} \big( d_X(x_i,x_j) -  d_Y(M^\top x_i, M^\top x_j) \big)^2$ & $\mathcal{O}^{d \times r}$ & $M^\top X$  \\

\\[-9pt] \\
LDA (\S \ref{sec:lda})&   $\frac{ \mathrm{tr}(M^\top \Sigma_B M)}{\mathrm{tr}(M^\top  \Sigma_W M)}$ & $\mathcal{O}^{d \times r}$ &   $M^\top X$  \\

\\[-11pt] \\
\parbox{1.9cm}{Traditional \\ CCA (\S \ref{sec:cca})}&  {\ssmall $\mathrm{ tr}\left( M_a^\top (X_a^{} X_a^\top)^{-1/2}X_a^{}X_b^\top (X_b^{} X_b^\top)^{-1/2} M_b^{} \right)$ }  & $\mathcal{O}^{d_a \times r} \times \mathcal{O}^{d_b \times r}$   & \parbox{2.5cm}{$M_a^\top\left(X_a^{}X_a^\top\right)^{-1/2}X_a^{}, $\\ $M_b^\top\left(X_b^{}X_b^\top\right)^{-1/2} X_b^{}$} \\

\\[-11pt] \\
\parbox{2cm}{Orthogonal \\ CCA (\S \ref{sec:cca})}&  $\frac{ \mathrm{tr}\left(M_a^\top X_a^{}X_b^\top M_b^{}\right)}{\sqrt{\mathrm{tr}\left(M_a^\top  X_a^{}X_a^\top M_a^{}\right)\mathrm{tr}\left(M_b^\top  X_b^{}X_b^\top M_b^{}\right)}}$ & $\mathcal{O}^{d_a \times r} \times \mathcal{O}^{d_b \times r}$   & $M_a^\top X_a^{}$ , $M_b^\top X_b^{}$ \\

\\[-9pt] \\
MAF (\S \ref{sec:maf}) &   $\frac{ \mathrm{tr}(M^\top \Sigma_\delta M)}{\mathrm{tr}(M^\top  \Sigma M)}$ & $\mathcal{O}^{d \times r}$ & $M^\top X$  \\

\\[-9pt] \\
SFA (\S \ref{sec:sfa}) &   $\mathrm{tr}(M^\top \dot{X}\dot{X}^\top M)$ & $\mathcal{O}^{d \times r}$ & $M^\top X$  \\

\\[-9pt] \\
SDR (\S \ref{sec:sdr}) &   $\mathrm{tr}\left( \bar{K}_Z \left( \bar{K}_{M^\top X} + n\epsilon I \right)^{-1} \right)$
 & $\mathcal{O}^{d \times r}$ & $M^\top X$  \\

\\[-9pt] \\
LPP (\S \ref{sec:lpp}) &   $\mathrm{tr}\left( M^\top (XDX^\top)^{-\top/2} XLX^\top (XDX^\top)^{-1/2} M \right)$
 & $\mathcal{O}^{d \times r}$ & $M^\top (XDX^\top)^{-\top/2} X$  \\

\\[-9pt] \\
UICA (\S \ref{sec:uica}) &{\scriptsize $\frac{1}{2} \log | M^\top M | + \frac{1}{n} \sum_{i=1}^n \sum_{k=1}^r \log f_\theta\left( m_k^\top x_n \right)$ } & $\reals^{d \times r}$ & $M^{\top}X$ \\

\\[-9pt] \\
PPCA (\S \ref{sec:ppca}) & {\scriptsize $\log| MM^{\top} + \sigma^2I | + \text{tr}\left(XX^{\top}(MM^{\top} + \sigma^2I )^{-1}\right)$ } & $\reals^{d \times r}$ & $M^{\top}(MM^{\top} + \sigma^2 I)^{-1}X$ \\

\\[-9pt] \\
FA  (\S \ref{sec:fa})&   {\scriptsize $\log| MM^{\top} + D | + \text{tr}\left(XX^{\top}(MM^{\top} + D )^{-1}\right)$ }& $\reals^{d \times r}$ & $M^{\top}(MM^{\top}+ D)^{-1}X $ \\

\\[-9pt] \\
LR (\S \ref{sec:lr})   &  $|| X_{out} - MX_{in}||_F^2 + \lambda || M||_p $ & $\reals^{d \times r}$ & $ SV^\top X_{in}$ for $M=USV^\top$ \\
\\[-11pt] \\

\\[-15pt] \\
DML (\S \ref{sec:dml})   & 
\parbox{5.9cm}{\centering  {\ssmall $\sum_{i,j \in \eta(i)}\Big\{ d_M(x_i,x_j)^2 +  \lambda  \sum_\ell \mathds{1}(z_i \neq z_\ell)$
\\ $  \left[ 1 + d_M(x_i,x_j)^2 - d_M(x_i,x_\ell)^2 \right]_+\Big\} $ } } & $\reals^{d \times r}$ &  $M^\top X$ \\
\\[-12pt] \\

\hline\\[-7pt]
\hline\\[-7pt]

\end{tabular}
\caption{Summary of linear dimensionality reduction methods.}
\label{tbl:results}
\end{table*}

\section{Results}
\label{sec:examples}

All methods considered here have specified $\mathcal{M}$ as either unconstrained matrices or matrices with orthonormal columns, variables in the space $\reals^{d \times r}$.   In the unconstrained case, numerous standard optimizers can and have been brought to bear to optimize the objective $f_X(M)$.  In the orthogonal case, we have also claimed that the very well-understood geometry of the manifold of orthogonal matrices enables optimization over these manifolds.    Pursuing such approaches is critical to consolidating and extending dimensionality reduction, as orthogonal projections $Y= M^\top X$ for $M \in \mathcal{O}^{d \times r}$ are arguably the most natural formulation of linear dimensionality reduction: one seeks a low-dimensional view of the data where some feature is optimally preserved.

The matrix family $\mathcal{O}^{d \times r}$ is precisely the real Stiefel manifold, which is a compact, embedded submanifold of $\reals^{d \times r}$.  In our context, this means that many important intuitions of optimization can be carried over onto the Stiefel manifold.  Notably, with a differentiable objective function $f_X(M)$ and its gradient $\nabla_M f$,  one can carry out standard first order optimization via a projected gradient method, where the unconstrained gradient is mapped onto the Stiefel manifold for gradient steps and line searches.   Second order techniques also exist, with some added complexity.    The foundations of these techniques are \cite{luenberger1972, gabay1982}, both of which build on classic and straightforward results from differential geometry.  More recently, \cite{edelman1998} sparked significant interest in optimization over matrix manifolds.  Some relevant examples include \cite{manton2002, manton2004, fiori2005, nishimori2005, abrudan2008, ulfarsson2008, srivastava2005, srivastava2010, varshney2011}.  Indeed, some of these works have been in the machine learning community \cite[]{fiori2005, ulfarsson2008, varshney2011}, and some have made the connection of geometric optimization methods to PCA \cite[]{srivastava2005, ulfarsson2008, srivastava2010, varshney2011}.   
The basic geometry of this manifold, as well as optimization over Riemannian manifolds, has been often presented and is now fairly standard.  For completeness, we include a primer on this topic in Appendix \ref{appdx:stiefel}.  There, as a motivating example, we derive the tangent space, the projection operation, and a retraction operation for the Stiefel manifold.   Appendix \ref{appdx:stiefel} then includes Algorithm \ref{alg:opt}, which uses these objects to present an optimization routine that performs gradient descent over the Stiefel manifold.  For a thorough treatment, we refer the interested reader to the excellent summary of much of this modern work \cite[]{absilBook}.

One important technical note warrants mention here.  The Stiefel manifold is the manifold of all ordered $r$-tuples of orthonormal vectors in $\reals^d$, but in some cases the dimensionality reduction objective $f_X(\cdot)$ evaluates only the subspace (orthonormal basis) implied by $M$, not the particular choice and order of the orthonormal vectors in $M$.  This class of objective functions is precisely those functions $f_X(M)$ such that, for any $r \times r$ orthogonal matrix $R$, $f_X(M) = f(MR)$.   The implied constraint in these cases is the manifold of rank-$r$ subspaces in $\reals^d$, which corresponds to the real Grassmann manifold $\mathcal{G}^{d \times r}$ (another very well understood manifold).  As a clarifying example, note that the PCA objective is redundant on the Stiefel manifold: if we want the highest variance $r$-dimensional projection of our data, the parameterization of those $r$ dimensions is arbitrary, and indeed $f(M) = || X - MM^\top X ||_F^2 = f(MR)$ for any orthogonal $R$.  If one is particularly interested in ranked eigenvectors, there are standard numerical tricks to break this equivalence and produce an ordered result: for example, maximizing $\text{tr}(AM^\top XX^\top M)$ over the Stiefel manifold, where $A$ is any diagonal matrix with ordered elements ($A_{11} > ... > A_{rr}$).  From the perspective of optimization and linear dimensionality reduction, the difference between the Grassmann and Stiefel manifold is one of identifiability.  Since there is an uncountable set of Stiefel points corresponding to a single Grassmann point, it seems sensible for many reasons to optimize over the Grassmann manifold when possible (though, as our results will show, this distinction empirically mattered very little).  Indeed, most of the optimization literature noted above also deals with the Grassmann case, and the techniques are similar.  Conveniently, an objective $f_X(M)$ can be quickly tested for the true implied manifold by comparing values of $f_X(MR)$ for various $R$.  Because the end result is still a matrix $M \in \mathcal{O}^{d \times r}$ (which happens to be in a canonical form in the Grassmann case), this fact truly is an implementation detail of the algorithm, not a fundamental distinction between different linear dimensionality reduction methods.  Thus, we present our results as agnostic to this choice, and we empirically revisit the question of identifiability at the end of this section.   

To demonstrate the effectiveness of these optimization techniques, we implemented a variety of linear dimensionality reduction methods with several solvers: first order steepest descent methods over the Stiefel and Grassmann manifolds, and second order trust region methods over the Stiefel and Grassmann manifolds \cite[]{absilBook}.  We implemented these methods in MATLAB, both natively for first order methods, and using the excellent {\tt manopt} software library \cite[]{boumal2014manopt} for first and second order methods (all code is available at {\tt http://github.com/cunni/ldr}).  All of these solvers accept, as input, data $X$ and any function that evaluates a differential objective $f_X(M)$ and its gradient $\nabla_M f$ at any point $M \in \mathcal{O}^{d \times r}$, and return, as output, an orthogonal $M$ that corresponds to a (local) optimum of the objective $f_X(M)$.  

\subsection{Example of Eigenvector Suboptimality}
\begin{figure}
\centering
\includegraphics[width=6.0in]{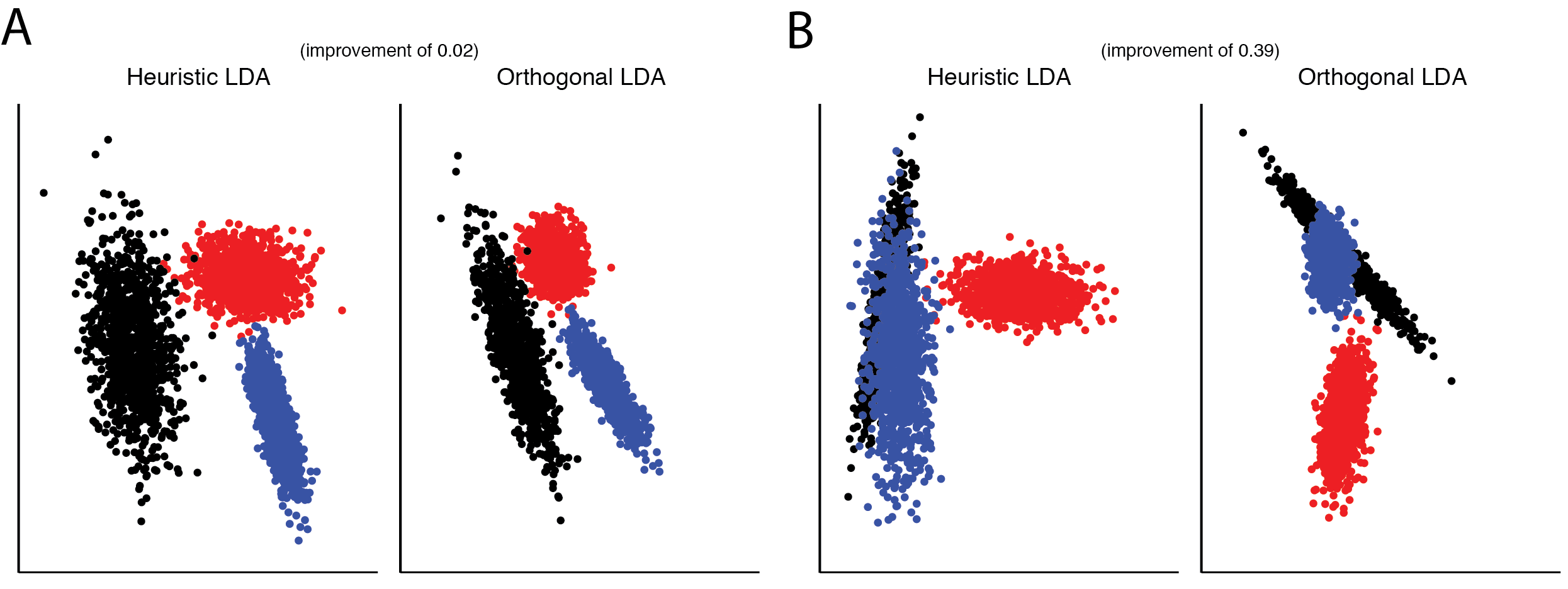}
\caption{Cautionary example of differences in objectives for LDA.  Panel A shows a data set that offers only marginal performance gain by using manifold optimization (Orthogonal LDA, right subpanel of panel A) rather than the traditional eigenvector heuristic (Heuristic LDA, left subpanel).  Panel B shows a data set that has a stark difference between the two methods.  The measured performance difference (see Equation \ref{eq:metric}) is shown.}
\label{fig:lda}
\end{figure}
We have cautioned throughout the above survey about the suboptimality of heuristic eigenvector solutions.  Figure \ref{fig:lda} demonstrates this suboptimality for LDA (Section \ref{sec:lda}).   In each panel (A and B), we simulated data of dimensionality $d=3$, with $n=3000$ points, corresponding to $1000$ points in each of $3$ clusters (shown in black, blue, and red).  Data in each cluster were normally distributed with random means (normal with standard deviation $5/2$) and random covariance (uniformly distributed orientation and exponentially distributed eccentricity with mean $5$).
 In the left subpanel of panel A, we then calculated the $r=2$ dimensional projection by orthogonalizing the top two eigenvectors of the matrix $\Sigma_W^{-1}\Sigma_B^{}$ (`Heuristic LDA').  In the right subpanel, we directly optimized the objective of Equation \ref{eq:lda} over $\mathcal{O}^{d \times r}$ (`Orthogonal LDA').  We calculate the normalized improvement of the manifold method as
\begin{equation}
\label{eq:metric}
-\frac{\left( f_X\left(M^{(orth)}\right) - f_X\left(M^{(eig)}\right) \right)}{ \left | f_X\left(M^{(eig)} \right) \right |}.
\end{equation}
Throughout the results we will call the results of traditional eigenvector approaches $M^{(eig)}$ and the results of our manifold solver $M^{(orth)}$.  Figure \ref{fig:lda}A shows an example where both the heuristic and manifold optimization methods return qualitatively similar results, and indeed the numerical improvement ($0.02$) reflects that indeed this heuristic is by no means wildly inappropriate for the stated objective.  Indeed, we know it to be correct for $r=1$.   Figure \ref{fig:lda}B shows a particularly telling example: both methods distinguish the red cluster easily, whereas the heuristic method confounds the black and blue clusters, while the optimization approach offers better separability, which indeed correlates with improvement on the stated objective of Equation \ref{eq:lda}.    It is critical to clarify the distinction between these two methods: the heuristic and orthogonal solutions are indeed optimal, but for \emph{different} objectives, as discussion in Section \ref{sec:lda}.   Thus, the purpose of this cautionary example is to highlight the importance of optimizing the intended objective, and the freedom to choose that objective without a tacit connection to a generalized eigenvalue problem.  These goals can be directly and generically achieved with the optimization framework of Equation \ref{eq:1}. 

\subsection{Performance Improvement} 
\begin{figure}
\centering
\includegraphics[width=6.0in]{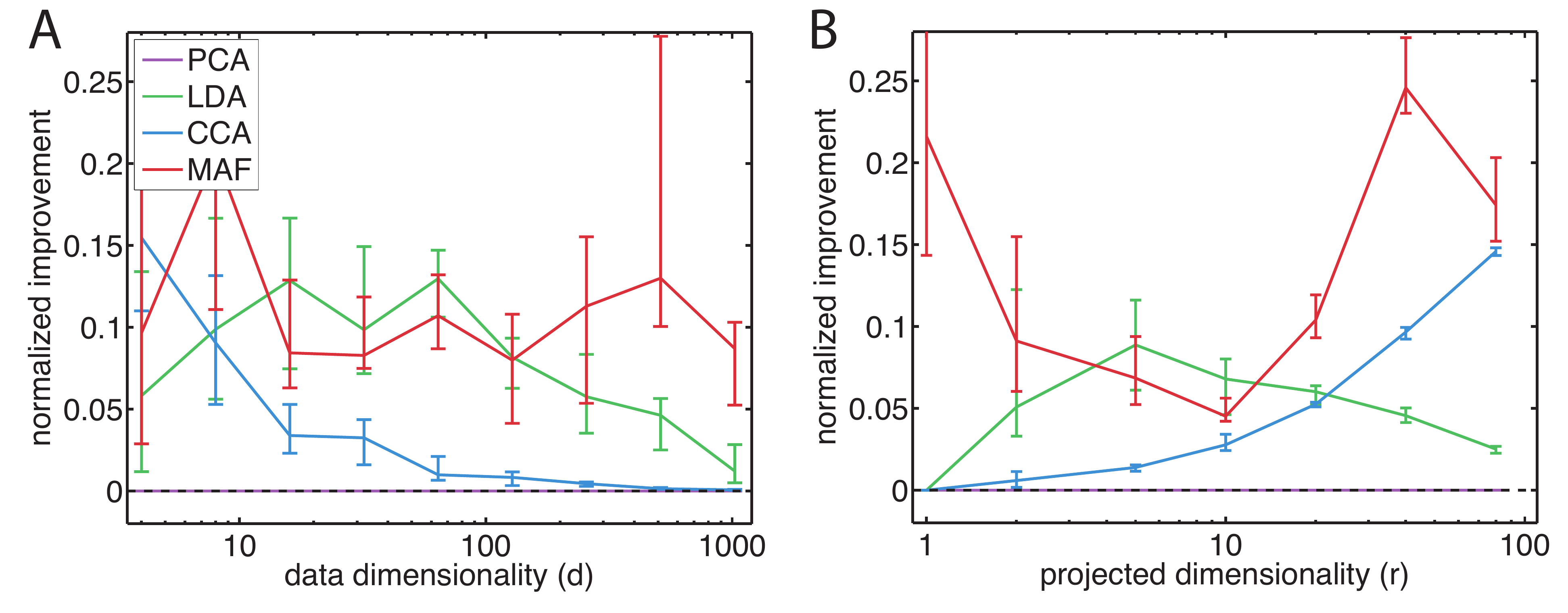}
\caption{Performance comparison between heuristic solvers and direct optimization of linear dimensionality reduction objectives.  The vertical axis denotes normalized improvement of the optimization program over traditional approaches.  
The error bars show median performance improvement and the central 50th percentile of $20$ independent runs at each choice of $(d,r)$.}
\label{fig:perf}
\end{figure}
Here we seek to demonstrate the quantitative improvements available by directly optimizing an objective, rather than resorting to an eigenvector heuristic.  First we implemented PCA (Section \ref{sec:pca}) using both methods.  We ran PCA on $20$ random data sets for each dimensionality $d \in \{4,8,16,...,1024\}$, each time projecting onto $r=3$ dimensions.  Data were normally distributed with random covariance (exponentially distributed eccentricity with mean $2$).  We calculated $f_X\left(M^{(eig)}\right)$ and $f_X\left(M^{(orth)}\right)$ from Equation \ref{eq:pca}, and we calculated the normalized improvement of the manifold method as above in Equation \ref{eq:metric}. Since the eigenvector decomposition is provably optimal for PCA, our method should demonstrate no improvement.  Indeed, Figure \ref{fig:perf} (purple trace) shows the distribution of normalized improvements for PCA is entirely $0$ in panel A.  We then repeated this analysis for a fixed data dimensionality $d=100$ (generating data as above), now ranging the projected dimensionality $r \in \{1,2,5,10,20,40,80 \}$.  These results are shown in Figure \ref{fig:perf}B, and again, the optimization approach recovers the known PCA optima precisely.  This confirmatory result also shows, pleasingly, that there is no empirical downside (in terms of accuracy) to using manifold optimization.

We next repeated the same experiment for LDA (Section \ref{sec:lda}).   We generated data with $1000$ data points in each of $d$ classes, where within class data was generated according to a normal distribution with random covariance (uniformly distributed orientation and exponentially distributed eccentricity with mean $5$), and each class mean vector was randomly chosen (normal with standard deviation $5/d$).  
We compared the suboptimal LDA heuristic $M^{(eig)}$ (orthogonalizing the top $r$ eigenvectors of $\Sigma_W^{-1}\Sigma_B^{}$) to the direct optimization of $f_X(M) = \mathrm{tr}( M^\top \Sigma_B M ) / \mathrm{tr}(M^\top \Sigma_W M )$, which produced $M^{(orth)}$.  Unlike in PCA, Figure \ref{fig:perf} (green traces) shows that directly addressing the LDA objective produces significant performance improvements.  The green trace is plotted at the median, and the error bars show the median $50\%$ of the distribution of performance improvements across both data dimensionality $d$ (panel A) and projected dimensionality $r$ (panel B).

We next implemented Traditional CCA and Orthogonal CCA as introduced in Section \ref{sec:cca}, which yield the blue performance distributions shown in Figure \ref{fig:perf}A and B.  Data set $X_a$ was generated by a random linear transformation of a latent data set $Z$ (iid standard normal points with dimensionality of $d/2$; the random linear transformation had the same distribution), plus noise, and data set $X_b$ was generated by a different random linear transformation of the same latent $Z$, plus noise.     Again we see significant improvement of direct Orthogonal CCA over orthogonalizing Traditional CCA, when evaluated under the correlation objective of Equation \ref{eq:cca_orth}.  First, we note that to be conservative in this case we omit the denominator term from the improvement metric (Equation \ref{eq:metric}); that is, we do not normalize CCA improvements. CCA has a correlation objective, which is already a normalized quantity, and thus renormalizing would increase these improvements.  More importantly, it is essential to note that we do not claim any suboptimality of Hotelling's Traditional CCA in solving Equation \ref{eq:cca_trad}.  Rather, it is the subsequent heuristic choice of orthogonalizing the resulting mapping that is problematic.  In other words, we show that if one seeks an orthogonal projection of the data, as is often desired in practice, one should do so directly.  Our CCA results demonstrate the substantial underperformance of eigenvector heuristics in this case, and our generic solver allows a direct solution without conceptual difficulty. 

Finally, we implemented MAF as introduced in Section \ref{sec:maf}, where we generated data by a random linear transformation (uniformly distributed entries on $[0,d^{-1/2}]$) of $d$ dimensions of univariate random temporal functions, which we generated with cubic splines with four randomly located knots (uniformly distributed in the domain, standard normally distributed in range), plus noise.  MAF is another method that has been solved using an eigenvector heuristic, and the performance improvement is shown in red in Figure \ref{fig:perf}. 

In total, Figure \ref{fig:perf} offers some key points of interpretation.  First, note that no data lie in the negative halfplane (see black dashed line atop the purple line at $0$).  Though unsurprising, this is an important confirmation that the optimization program performs unambiguously better than or equal to heuristic methods.  Second, methods other than PCA produce approximately 10\% improvement using direct optimization, a significant improvement that suggests the broad use of this optimization framework.
Third, a natural question for these nonconvex programs is that of local optima.  We found that, across a wide range of choices for $d$ and $r$, nearly all methods converged to the same optimal value whether started at a random $M$ or started at the heuristic point $M^{(eig)}$.   Deeper characterization of local optima should be highly dependent on the particular objective and is beyond the scope of this work.
Third, we note that methods sometimes have performance equal to the heuristic method; indeed  $M^{(eig)}$ is sometimes a local optimum.  We found empirically that larger $r$ makes this less likely, and larger $d$ makes this more likely. 

A significant point of interpretation is that of size of average performance.  We stress that these data sets were not carefully chosen to demonstrate effect.  Indeed, we are able to adversarially choose data to create much larger performance improvements, and similarly we can choose data sets that demonstrate no effect.   Thus, one should not infer from Figure \ref{fig:perf} that, for example, Orthogonal CCA fundamentally has increasing benefit over the heuristic approach with increasing $r$ (or decreasing benefit with increasing $d$).  Instead, we encourage the takeaway of this performance figure to be that one should always optimize the objective of interest directly, rather than resorting to a reasonable but theoretically unsound eigenvector heuristic, as the performance loss is potentially meaningful.     

\subsection{Computational Cost}
\begin{figure}[t]
\centering
\includegraphics[width=6.0in]{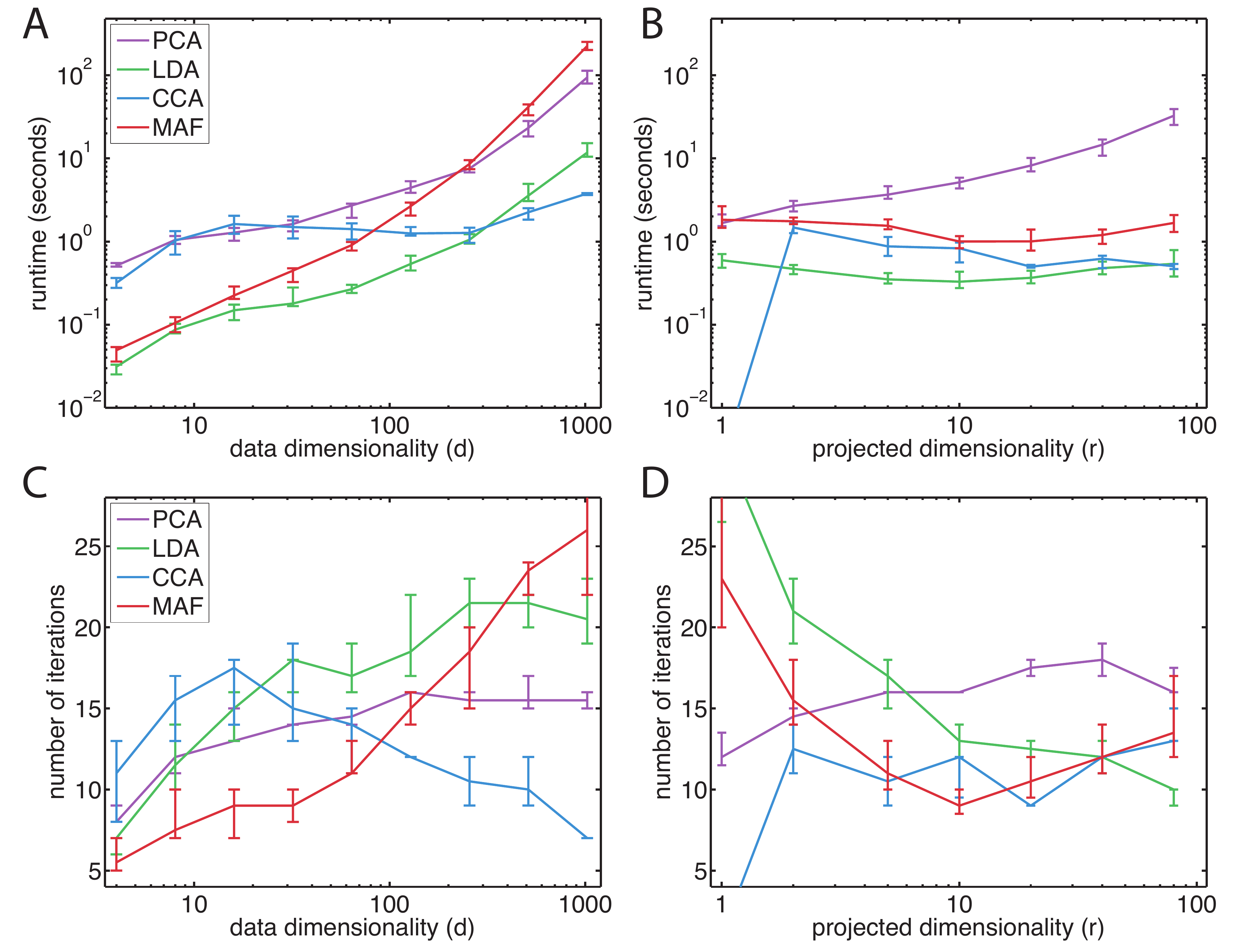}
\caption{Computational cost of direct optimization of linear dimensionality reduction objectives.  Data sets are the same as those in Figure \ref{fig:perf}.  The vertical axis in panels A and B denotes runtime in seconds.  Panels C and D show the  same data by the number of solver iterations.}
\label{fig:comp}
\end{figure}
Importantly, this matrix manifold solver does not incur massive computational cost.  The only additional computation beyond standard unconstrained first-order optimization of $dr$ variables is the projection onto or along the manifold to ensure a feasible $M \in \mathcal{O}^{d\times r}$, which in any scheme requires a matrix decomposition (see Appendix \ref{appdx:stiefel}).  Thus each algorithmic step carries an additional cost of $O(dr^2)$.   This cost is in many cases dwarfed by the larger cost of calculating matrix-matrix products with a data matrix $X \in \reals^{d \times n}$ (which often appear in the gradient calculations $\nabla_M f$).  Second order methods approximate or evaluate a Hessian, which incurs more complexity per iteration, but as usual at the tradeoff of drastically fewer iterations.   Accordingly, the runtime of manifold optimization is at worst moderately degraded compared to an unconstrained first or second order method.  Compared to eigenvector heuristics, which if implemented as a compact SVD cost only $O(dr^2)$, direct optimization is an order of magnitude or more slower due to the iterative nature of the algorithm.

Figure \ref{fig:comp} shows the computational cost of these methods, using the same data as in the previous section.  In Figure \ref{fig:comp}A, at each of $d \in \{ 4,8,16,...,1024\}$ and for $r=3$, we ran PCA, LDA, CCA, and MAF 20 times, and we show here the median and central $50\%$ of the runtime distribution (in seconds). 
This panel demonstrates that runtime increases approximately linearly as expected in $d$: runtime increases by approximately three orders of magnitude over three orders of magnitude increase in $d$.  We do a similar simulation in Figure \ref{fig:comp}B at each of $r \in \{1,2,5,10,20,40,80\}$ for a fixed $d=100$, and again runtime is increasing.  

Figures \ref{fig:comp}C and \ref{fig:comp}D show the same data as in Figures \ref{fig:comp}A and \ref{fig:comp}B, but by number of solver iterations.  In this figure we used a second-order solver over the Grassmann manifold in PCA, LDA, and MAF (critically, the same solver for all three), and the second-order solver over the product of two Stiefel manifolds in the case of CCA.  These two panels again underscore the overall point of Figure \ref{fig:comp}:  runtime complexity is not particularly burdensome across a range of reasonable choices for $d$ and $r$, even with a generic solver.

\subsection{Choice of Solver, and Identifiability}
\begin{figure}[t]
\centering
\includegraphics[width=6.0in]{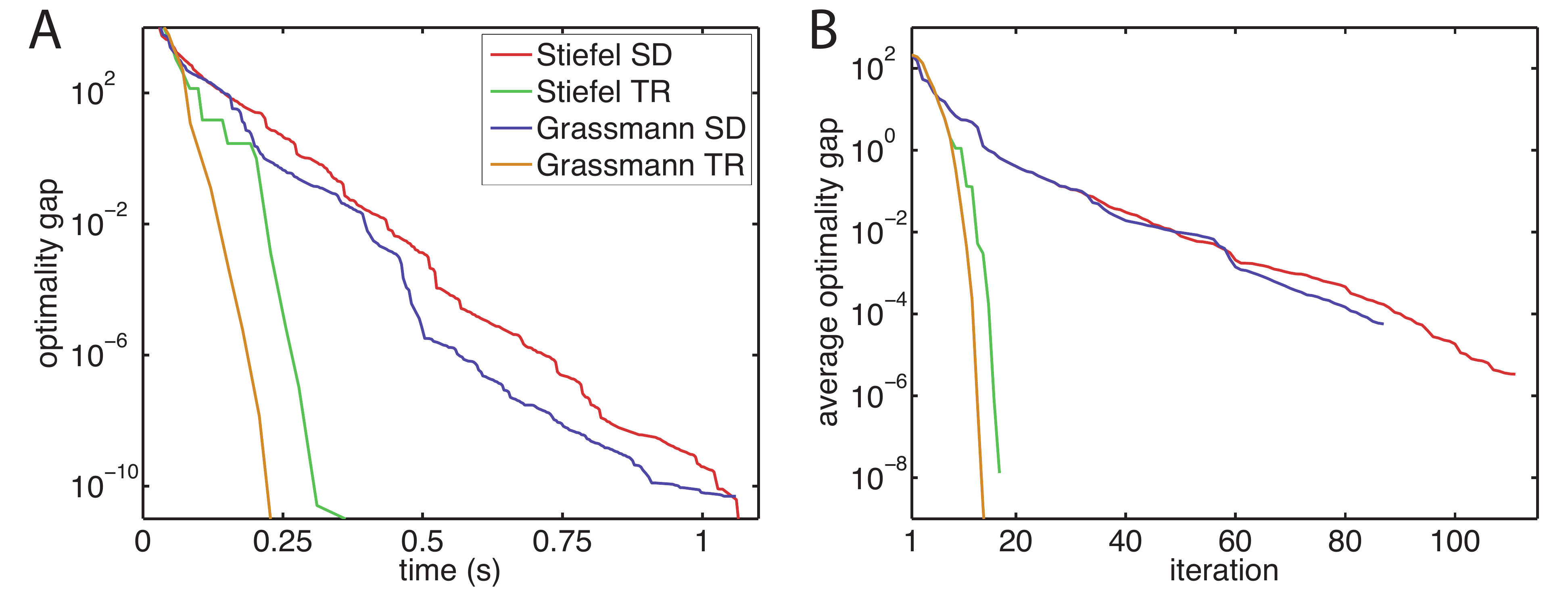}
\caption{Comparison of different optimization techniques. PCA was run on $100$ independent data sets of size $d=100$, projecting to $r=10$ dimensions.  Panel A shows the median runtime performance (across data sets), with optimality gap as a function of runtime in seconds.  Panel B shows the average optimality gap by iteration.  PCA was run on each of these data sets independently with the Stiefel steepest descent (red), Stiefel trust region (green), Grassmann steepest descent (blue), and Grassmann trust region (brown) solvers.}
\label{fig:converge}
\end{figure}
We have claimed that the choice of optimization over the Stiefel or Grassmann manifold is a question of identifiability, and further that empirically it seems to matter little to algorithmic performance.  Figure \ref{fig:converge} gives evidence to that claim.  We created $100$ independent data sets with $d = 100$ and $r=10$ for PCA.  Here the choice of algorithm is less important, and PCA is a sensible choice because we know the global optimum.  We ran PCA using four solvers: first-order steepest descent over the Stiefel manifold, first-order steepest descent over the Grassmann manifold,  second-order trust region optimization over the Stiefel manifold, and second-order trust region optimization over the Grassmann manifold.  Figure \ref{fig:converge} shows the optimality gap by solver choice for each of these four solvers.  Figure \ref{fig:converge}A shows the optimality gap as a function of time for the median performing solver (median across the $100$ independent data sets), and Figure \ref{fig:converge}B shows the optimality gap (mean across all the $100$ independent data sets) as a function of algorithmic iteration.  From these figures it is clear that second-order methods outperform first order methods, though perhaps less than one might typically expect.  More importantly, the difference between the choice of optimization over the Stiefel or Grassmann manifold is minor at best.     This figure, along with previous results, suggest the feasibility of a generic solver for orthogonal linear dimensionality reduction.  
%

\section{Discussion}

Dimensionality reduction is a cornerstone of data analysis.   Among many methods, perhaps none are more often used than the linear class of methods.  By considering these methods as optimization programs of user-specified objectives over orthogonal or unconstrained matrix manifolds, we have surveyed a surprisingly fragmented literature, offered insights into the shortcomings of traditional eigenvector heuristics, and have pointed to straightforward generalizations with an objective-agnostic linear dimensionality reduction solver.  The results of Section \ref{sec:examples} suggest that linear dimensionality reduction can be abstracted away in the same way that unconstrained optimization has been, as a numerical technology that can sometimes be treated as a black-box solver.   This survey also suggests that future linear dimensionality reduction algorithms can be derived in a simpler and more principled fashion.  Of course, even with such a method one must be careful to design a linear dimensionality reduction sensibly to avoid the many unintuitive pitfalls of high-dimensional data (e.g., \citealp{diaconis1984asymptotics}).

Other authors have surveyed dimensionality reduction algorithms.  Some relevant examples include \cite{burges2010, delatorre2012, sun2009, borga1997}.  These works all focus on particular subsets of the dimensionality reduction field, and our work here is no different, insomuch as we focus exclusively on linear dimensionality reduction and the connecting concept of optimization over matrix manifolds.  \cite{burges2010} gives an excellent tutorial review of popular methods, including both linear and nonlinear methods, dividing those methods into projective and manifold approaches.  \cite{delatorre2012}  surveys five linear and nonlinear methods with their kernelized counterparts using methods from kernel regression.  \cite{borga1997} and \cite{sun2009} focus on those methods that can be cast as generalized eigenvalue problems, and derive scalable algorithms for those methods, connecting to the broad literature on optimizing Rayleigh quotients.    

The simple optimization framework discussed herein offers a direct approach to linear dimensionality reduction: many linear dimensionality reduction methods seek a meaningful, low-dimensional orthogonal subspace of the data, so it is natural to create a program that directly optimizes some objective on the data over these subspaces.  This claim is supported by the number of linear dimensionality reduction methods that fit naturally into this framework, by the ease with which new methods can be created, and by the significant performance gains achieved with direct optimization.   Thus we believe this survey offers a valuable simplifying principle for linear dimensionality reduction.

This optimization framework is conceptually most similar to the projection index from important literature in projection pursuit \cite[]{huber1985projection, friedman1987exploratory}: both that literature and the present work focus on optimizing objective functions on projections to a lower dimensional coordinate space.   Since the time of the fundamental work in projection pursuit, massive developments in computational power and advances in optimization over matrix manifolds suggest the merit of the present approach.  First, the projection pursuit literature is inherently greedy: univariate projections are optimized over the projection index, that structure is removed from the high dimensional data, and the process is repeated.  This approach leads to (potentially significant) suboptimality of the results and requires costly computation on the space of the high-dimensional data for structure removal.   The present matrix manifold framework circumvents both of these issues.  Thus, while the spirit of this framework is very much in line with the idea of a projection index, this framework, both in concept and in implementation, is critically enabled by tools that were unavailable to the original development of projection pursuit.    


\acks{JPC and ZG received funding from the UK Engineering and Physical Sciences Research Council (EPSRC EP/H019472/1). JPC received funding from a Sloan Research Fellowship, the Simons Foundation (SCGB\#325171 and SCGB\#325233), the Grossman Center at Columbia University, and the Gatsby Charitable Trust. }


\appendix

\section{Optimization over the Stiefel Manifold}
\label{appdx:stiefel}

Here we offer a basic introduction to optimization over matrix manifolds, restricting our focus to a first-order, projected gradient optimization over the Stiefel manifold $\mathcal{O}^{d \times r}$.  Intuitively, manifold projected gradient methods are iterative optimization routines that require firstly an understanding of search directions along the constraint set, called the tangent space (\S \ref{sec:tan}).  With an objective $f$, gradients $\nabla_M f$ are then calculated in the full space, in this case $\reals^{d \times r}$.  These gradients are projected onto that tangent space  (\S \ref{sec:proj}).  Any nonzero step in a linear tangent space will depart from the nonlinear constraint set, so finally a \emph{retraction} is needed to map a step onto the constraint set (\S \ref{sec:retract}).  With these three components, a standard first-order iterative solver can be carried out, with typical convergence guarantees.  We conclude this tutorial appendix with pseudocode in \S \ref{sec:algo} and a figure summarizing these steps (Figure \ref{fig:stiefel}).           

We have previously introduced the Stiefel manifold $\mathcal{O}^{d\times r}$ as the set of all matrices with orthonormal columns, namely
$\mathcal{O}^{d \times r} = \left\{ M \in \reals^{d \times r} : M^\top M = I \right\},$
where $I$ is the $r \times r$ identity matrix.    $\mathcal{O}^{d \times r}$ is a manifold, an embedded submanifold of $\reals^{d \times r}$, and bounded and closed (and thus compact).  From these facts we can carry over all intuitions of an explicit (though nonlinear and nonconvex) constraint set within $\reals^{d \times r}$.  

\subsection{Tangent Space $T_M \mathcal{O}^{d \times r}$}
\label{sec:tan}
Critical to understanding the geometry of any manifold (in particular to exploit that geometry for optimization) is the \emph{tangent space}, the linear (vector space) approximation to the manifold at a particular point.  To define this space, we first define a \emph{curve} on the manifold $\mathcal{O}^{d\times r}$ as a smooth map $\gamma(\cdot): \reals \rightarrow \mathcal{O}^{d\times r}$.  Then, the tangent space is
\begin{equation}
\label{eq:TM}
T_M \mathcal{O}^{d \times r} = \left\{ \dot{\gamma}(0) : \gamma(\cdot) \text{~is~a~curve~on~} \mathcal{O}^{d\times r} \text{ with }  \gamma(0) = M \right\},
\end{equation}
\noindent where $\dot{\gamma}$ is the derivative $\frac{d}{dt}\gamma(t)$.  Loosely, $T_M \mathcal{O}^{d \times r}$ is the space of directions along the manifold at a point $M$.  While Equation \ref{eq:TM} is fairly general for embedded submanifolds, it is abstract and leaves little insight into numerical implementation.  Conveniently, the tangent space of the Stiefel manifold has a particularly nice equivalent form.   
\begin{clm}[Tangent space of the Stiefel Manifold]
\label{clm:tan}
The following sets are equivalent:
\begin{eqnarray}
\label{eq:t0}
T_M \mathcal{O}^{d \times r} &= &\left\{ \dot{\gamma}(0) : \gamma(\cdot) \mathrm{~is~a~curve~on~}  \mathcal{O}^{d\times r} \mathrm{~with~}  \gamma(0) = M \right\}, \\
\label{eq:t1}
T_1 & = & \left\{ X \in \reals^{d \times r} : M^\top X + X^\top M = 0 \right\}, \\
\label{eq:t2}
T_2 & = & \left\{ MA + (I-MM^\top)B : A = -A^\top , B \in \reals^{d \times r}  \right\}.
\end{eqnarray}
\end{clm}

\begin{proof}
The proof proceeds in four steps:
\begin{enumerate}
\item $X \in T_M \mathcal{O}^{d \times r}  \Rightarrow X \in T_1$

Considering a curve $\gamma(t)$ from Equation \ref{eq:t0}, we know $\gamma(t)^\top \gamma(t) = I$ (every point of the curve is on the manifold).  We differentiate in $t$ to see $\gamma(t)^\top \dot{\gamma}(t) + \dot{\gamma}(t)^\top \gamma(t) = 0$.  At $t=0$, we have $\gamma(0)=M$, and we define the tangent space element $\dot{\gamma}(0) = X$.  Then $X$ is such that $M^\top X + X^\top M = 0$.  

\item $X \in T_1 \Rightarrow X \in T_M \mathcal{O}^{d \times r}$

We must construct a curve such that any $X \in T_1$ is a point in the tangent space; consider  $\gamma(t) = (M + tX)( I + t^2X^\top X )^{-1/2}$ (a choice that we will see again below in \S\ref{sec:retract}).  First, this curve satisfies $\gamma(0) = M$.  Second, $\gamma(\cdot)$ is a curve on the Stiefel manifold, since every point $\gamma(t)$ satisfies
\begin{eqnarray*}
\gamma(t)^\top\gamma(t) & = & ( I + t^2X^\top X )^{-1/2} (M + tX)^\top (M + tX)( I + t^2X^\top X )^{-1/2} \\
& = & ( I + t^2X^\top X )^{-1/2} (M^\top M + tM^\top X + tX^\top M + t^2 X^\top X ) ( I + t^2X^\top X )^{-1/2} \\
& = & ( I + t^2X^\top X )^{-1/2} (I +  t^2 X^\top X ) ( I + t^2X^\top X )^{-1/2} \\
& = & I,
\end{eqnarray*}
where the third line uses $M\in \mathcal{O}^{d\times r}$ and $X \in T_1$.  It remains to show only that $\dot{\gamma}(0) = X$.  We differentiate $\gamma(t)$ as 
\begin{equation}
\label{eq:dgam}
\dot{\gamma}(t) = X(I + t^2X^\top X)^{-1/2} + (M + tX) \frac{d}{dt}(I + t^2 X^\top X)^{-1/2}.
\end{equation}
The rightmost derivative term of Equation \ref{eq:dgam} does not have a closed form, but is the unique solution to a Sylvester equation.  Letting $\alpha(t) = (I + t^2 X^\top X)^{-1/2}$, we seek $\dot{\alpha}(0)$.  By implicit differentiation, 
\begin{eqnarray*}
\left[ \frac{d}{dt}\alpha(t)\alpha(t)\right]_{t=0} & = & \left[ \frac{d}{dt}(I + t^2 X^\top X)^{-1} \right]_{t=0}\\
\dot{\alpha}(0)\alpha(0) + \alpha(0)\dot{\alpha}(0) & = & \left[ (I + t^2 X^\top X)^{-1}\left( 2tX^\top X \right) (I + t^2 X^\top X)^{-1}  \right]_{t=0}\\
2\dot{\alpha}(0) & = & 0,\\
\end{eqnarray*}
since $\alpha(0) = I$.  Thus we see $\dot{\alpha}(0) = \left[\frac{d}{dt}(I + t^2 X^\top X)^{-1/2}\right]_{t=0} = 0$.  Equation \ref{eq:dgam} yields $\dot{\gamma}(0) = X$, which completes the proof of the converse.

\item $X \in T_2 \Rightarrow X \in T_1$

Let $X = MA + (I-MM^\top)B$ according to Equation \ref{eq:t2}.  Then
\begin{eqnarray*}
M^\top X + X^\top M & = & M^\top MA + M^\top (I-MM^\top)B + A^\top M^\top M + B^\top(I-MM^\top) M \\
& = & A + A^\top \\
& = & 0,
\end{eqnarray*}
by the skew-symmetry of $A$ and $M \in \mathcal{O}^{d \times r}$.  

\item $X \in T_1 \Rightarrow X \in T_2$

We show the transposition $X \not\in T_2 \Rightarrow X \not\in T_1$.  By the definition of $T_2$, $X = MA + (I-MM^\top)B$ is not in $T_2$ if and only if $A \neq -A^\top$.  Then, using the previous argument, we see that such an $X$ has $M^\top X + X^\top M \neq 0$, and thus is not a member of $T_1$.  
\end{enumerate}
Thus, the three tangent space definitions Equations \ref{eq:t0}-\ref{eq:t2} are equivalent.  The definition of Equation \ref{eq:t2} is particularly useful as it is constructive, which is essential when considering optimization.
\end{proof}
%
\subsection{Projection $\pi_M : \reals^{d \times r} \rightarrow T_M \mathcal{O}^{d \times r}$}
\label{sec:proj}
Because $\mathcal{O}^{d \times r}$ is an embedded submanifold of $\reals^{d \times r}$, it is natural to consider the metric implied by Euclidean space : $\reals^{d \times r}$ endowed with the standard inner product $\langle P , N \rangle = \mathrm{tr}(P^\top N)$, and the induced Frobenius norm $||\cdot ||_F$.  With this metric, the Stiefel manifold is then a Riemannian submanifold of Euclidean space. This immediately allows us to consider the projection of an arbitrary vector $Z\in \reals^{d \times r}$ onto the tangent space $T_M \mathcal{O}^{d \times r}$, namely
\begin{eqnarray*}
\pi(Z) & = & \underset{X \in T_M \mathcal{O}^{d \times r}}{\argmin} || Z - X ||_F \\
& = &  {\argmin}~ || Z - (MA - (I-MM^\top)B )||_F \\
& = & {\argmin}~ || (MM^\top Z -  MA) + (I-MM^\top)(Z - B )||_F \\
& = & {\argmin}~ || M(M^\top Z -  A)||_F + ||(I-MM^\top)(Z - B )||_F \\
& = & {\argmin}~ || M^\top Z -  A||_F + ||(I-MM^\top)(Z - B )||_F,
\end{eqnarray*}
where the last equality comes from the unitary invariance of the Frobenius norm.  This expression is minimized by setting $B=Z$ and setting $A$ to be the skew-symmetric part of $M^\top Z$, namely $A :=  \text{skew}( M^\top Z ) = \frac{1}{2}( M^\top Z - Z^\top M ) $ \cite[]{fan1955}.  This results in the projection
\begin{equation}
\label{eq:pi}
\pi_M(Z) = M \text{skew}(M^\top Z) + (I-MM^\top)Z.
\end{equation}

We note that an alternative canonical metric is often considered in this literature, namely  $\langle P , N \rangle_M = \mathrm{tr}\left(P^\top(I-MM^\top) N\right)$.  The literature is divided on this choice; for simplicity we choose the standard inner product.
%
\subsection{Retraction $r_M :  T_M \mathcal{O}^{d \times r} \rightarrow \mathcal{O}^{d \times r}$}
\label{sec:retract}
Projected gradient methods seek an iterative step in the direction of steepest descent along the manifold, namely $M + \beta \pi_M\left( - \nabla_M f \right)$.  For any nonzero step size $\beta$, this iterate will leave the Stiefel manifold.  Thus, a \emph{retraction} is required to map onto the manifold.  A number of projective retractions are available \cite[]{kaneko2013empirical}; here we define the retraction of a step $Z$ away from a current manifold point $M$ as
\begin{equation}
\label{eq:r}
r_M(Z) = \argmin_{N \in \mathcal{O}^{d \times r}} || N - (M +Z) ||_F,
\end{equation}
that is, the closest point on the manifold to the desired iterate $M+Z$.  For unitarily invariant norms, a classic result is that $r_M(Z) = UV^\top$, where $(M+Z) = USV^\top$ is the singular value decomposition \cite[]{fan1955}, or equivalently, $r_M(Z) = W$ for a polar decomposition  $(M+Z) = WP$.  Conveniently, when $Z \in T_M \mathcal{O}^{d \times r}$, this retraction has the simple closed form $r_M(Z) = (M + Z)(I + Z^\top Z)^{-1/2}$ \cite[]{kaneko2013empirical}, which explains the choice of curve in \S\ref{sec:tan}.

In the cases of the Stiefel and Grassmann manifolds, it is possible to directly calculate a manifold geodesic (shortest path between two points in the manifold).   While more aesthetically pleasing, calculating such a geodesic requires a matrix exponential, and thus has similar computational burden as a projective retraction (often the exponential is slightly more expensive).  Empirically, we have found very little difference in the convergence or computational burden of this choice, and thus we focus this tutorial on the conceptually simpler retraction.  \cite{absil2012projection} discuss projective retractions compared with geodesics/exponential maps.
\begin{figure}[t]
\centering
\includegraphics[width=2.8in]{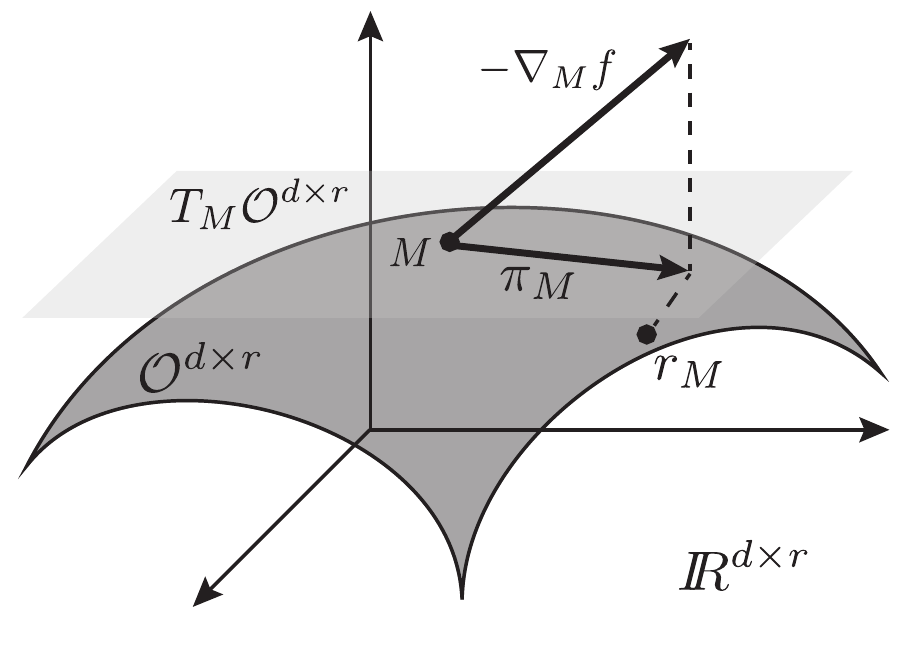}
\caption{Cartoon of a projected gradient step on the Stiefel manifold. Notation follows Algorithm \ref{alg:opt}.}
\label{fig:stiefel}
\end{figure}
%

\subsection{Pseudocode for a Projected Gradient Solver}
\label{sec:algo}
Algorithm \ref{alg:opt} gives pseudocode for a projected gradient method over the Stiefel manifold.  This generic algorithm requires only a choice of convergence parameters and a line search method, choices which are standard for first-order optimization. Chapter 4 of \cite{absilBook} offers a global convergence proof for such a method using Armijo line search.  Indeed, the only particular consideration for this algorithmic implementation is the tangent space $T_M \mathcal{O}^{d \times r}$, the projection $\pi_M$, and the retraction $r_M$.  
%
\begin{algorithm}
\caption{Gradient descent over the Stiefel manifold (with line search and retraction)}
\label{alg:opt}
\begin{algorithmic}[1]
\STATE initialize $M \in \mathcal{O}^{d \times r}$ 
\WHILE{$f(M)$ has not converged}
\STATE calculate $\nabla_M f \in \reals^{d \times r}$ \hfill \# free gradient of objective
\STATE calculate $\pi_M( -\nabla_M f) \in T_M \mathcal{O}^{d \times r}$ \hfill \# search direction (Equation \ref{eq:pi})  
\WHILE{$f( r_M(\beta \pi_M( -\nabla_M f) ) )$ is not sufficiently smaller than $f(M)$}  
\STATE adjust step size $\beta$ \hfill \# line search (using retraction, Equation \ref{eq:r})
\ENDWHILE
\STATE $M \leftarrow r_M( \beta \pi_M( -\nabla_M f) )$ \hfill \# iterate
\ENDWHILE
\STATE \textbf{return} (local) minima $M^*$ of $f$.
\end{algorithmic}
\end{algorithm}

It is worth noting that the above operations imply a two-stage gradient step: Algorithm \ref{alg:opt} first projects the free gradient onto the tangent space ($\pi_M$), and second the proposed step is retracted onto the manifold ($r_M$).   It is natural to ask why one does not perform this projection in one step, for example by projecting the free gradient directly onto the manifold.  Firstly, while there is a rich literature of such `one-step' projected gradient methods \cite[]{bertsekas1976goldstein}, convergence guarantees only exist for convex constraint sets.  Indeed, all matrix manifolds we have discussed are nonconvex (except the trivial $\reals^{d\times r}$).  The theory of convergence for nonconvex manifolds requires this two-step procedure.  Secondly, in our empirical experience, while a one-step projection method does often converge, that convergence is typically much slower than Algorithm 1.

%
This basic algorithm is extended in two ways: first, the constraint manifold $\mathcal{M}$ is taken to be the Grassmann manifold or some other manifold structure (like the product of Stiefel manifolds, as in CCA above); and second,  conjugate gradient methods or second-order optimization techniques can be similarly adapted to the setting of matrix manifolds.  Beyond these steps, understanding optimization over matrix manifolds in full generality requires topological and differential geometric machinery that is beyond the scope of this work.  All of these topics are discussed in the key reference to this appendix \cite[]{absilBook}, as well as the literature cited throughout this paper.



\bibliography{ldr}

\begin{thebibliography}{133}
\providecommand{\natexlab}[1]{#1}
\providecommand{\url}[1]{\texttt{#1}}
\expandafter\ifx\csname urlstyle\endcsname\relax
  \providecommand{\doi}[1]{doi: #1}\else
  \providecommand{\doi}{doi: \begingroup \urlstyle{rm}\Url}\fi

\bibitem[Abrudan et~al.(2008)Abrudan, Eriksson, and Koivunen]{abrudan2008}
T.~E. Abrudan, J.~Eriksson, and V.~Koivunen.
\newblock Steepest descent algorithms for optimization under unitary matrix
  constraint.
\newblock \emph{IEEE Transactions on Signal Processing}, 56:\penalty0
  1134--1147, 2008.

\bibitem[Absil and Malick(2012)]{absil2012projection}
P.~Absil and J.~Malick.
\newblock Projection-like retractions on matrix manifolds.
\newblock \emph{SIAM Journal on Optimization}, 22\penalty0 (1):\penalty0
  135--158, 2012.

\bibitem[Absil et~al.(2008)Absil, Mahony, and Sepulchre]{absilBook}
P.~A. Absil, R.~Mahony, and R.~Sepulchre.
\newblock \emph{Optimization Algorithms on Matrix Manifolds}.
\newblock Princeton University Press, Princeton, 2008.

\bibitem[Adragni and Cook(2009)]{adragni2009}
K.~P. Adragni and D.~Cook.
\newblock Sufficient dimension reduction and prediction in regression.
\newblock \emph{Philosophical Transactions of the Royal Society A},
  367:\penalty0 4385--4405, 2009.

\bibitem[Amari(1999)]{amari1999}
S.~Amari.
\newblock Natural gradient learning for over and under complete bases in {ICA}.
\newblock \emph{Neural Computation}, 11:\penalty0 1875--1883, 1999.

\bibitem[Baccini et~al.(1996)Baccini, Besse, and de~Faguerolles]{baccini1996}
A.~Baccini, P.~Besse, and A.~de~Faguerolles.
\newblock A {L}1-norm {PCA} and heuristic approach.
\newblock In \emph{Proceedings of the International Conference on Ordinal and
  Symbolic Data Analysis}, pages 359--368, 1996.

\bibitem[Bach et~al.(2011)Bach, Jenatton, Mairal, and
  Obozinski]{bach2011convex}
F.~Bach, R.~Jenatton, J.~Mairal, and G.~Obozinski.
\newblock Convex optimization with sparsity-inducing norms.
\newblock \emph{Optimization for Machine Learning}, pages 19--53, 2011.

\bibitem[Bak{\i}r et~al.(2004)Bak{\i}r, Gretton, Franz, and
  Sch{\"o}lkopf]{bakir2004multivariate}
G.~H. Bak{\i}r, A.~Gretton, M.~Franz, and B.~Sch{\"o}lkopf.
\newblock Multivariate regression via {S}tiefel manifold constraints.
\newblock In \emph{Pattern Recognition}, pages 262--269. Springer, 2004.

\bibitem[Bar-Hillel et~al.(2003)Bar-Hillel, Hertz, Shental, and
  Weinshall]{bar2003learning}
A.~Bar-Hillel, T.~Hertz, N.~Shental, and D.~Weinshall.
\newblock Learning distance functions using equivalence relations.
\newblock In \emph{Proceedings of the International Conference on Machine
  Learning}, volume~3, pages 11--18, 2003.

\bibitem[Belkin and Niyogi(2003)]{belkin2003}
M.~Belkin and P.~Niyogi.
\newblock Laplacian eigenmaps for dimensionality reduction and data
  representation.
\newblock \emph{Neural Computation}, 15\penalty0 (6):\penalty0 1373--1396,
  2003.

\bibitem[Bertsekas(1976)]{bertsekas1976goldstein}
D.~P. Bertsekas.
\newblock On the goldstein-levitin-polyak gradient projection method.
\newblock \emph{IEEE Transactions on Automatic Control}, 21\penalty0
  (2):\penalty0 174--184, 1976.

\bibitem[Bishop(2006)]{bishopBook}
C.~M. Bishop.
\newblock \emph{Pattern Recognition and Machine Learning}.
\newblock Springer, New York, 2006.

\bibitem[Borg and Groenen(2005)]{borgBook}
I.~Borg and P.~J. Groenen.
\newblock \emph{Modern Multidimensional Scaling: Theory and Applications}.
\newblock Springer Verlag, 2005.

\bibitem[Borga et~al.(1997)Borga, Landelius, and Knutsson]{borga1997}
M.~Borga, T.~Landelius, and H.~Knutsson.
\newblock A unified approach to {PCA}, {PLS}, {MLR}, and {CCA}.
\newblock \emph{Technical Report}, 1997.

\bibitem[Boumal et~al.(2014)Boumal, Mishra, Absil, and
  Sepulchre]{boumal2014manopt}
N.~Boumal, B.~Mishra, P.~Absil, and R.~Sepulchre.
\newblock Manopt, a matlab toolbox for optimization on manifolds.
\newblock \emph{Journal of Machine Learning Research}, 15:\penalty0 1455--1459,
  2014.

\bibitem[Boyd et~al.(2011)Boyd, Parikh, Chu, Peleato, and
  Eckstein]{boyd2011distributed}
S.~Boyd, N.~Parikh, E.~Chu, B.~Peleato, and J.~Eckstein.
\newblock Distributed optimization and statistical learning via the alternating
  direction method of multipliers.
\newblock \emph{Foundations \& Trends in Machine Learning}, 3\penalty0
  (1):\penalty0 1--122, 2011.

\bibitem[Bray and Martinez(2002)]{bray2002}
A.~Bray and D.~Martinez.
\newblock Kernel-based extraction of slow features: Complex cells learn
  disparity and translation invariance from natural images.
\newblock In \emph{Advances in Neural Information Processing Systems}, pages
  253--260, 2002.

\bibitem[Brendel et~al.(2011)Brendel, Romo, and Machens]{brendel2011demixed}
W.~Brendel, R.~Romo, and C.~K. Machens.
\newblock Demixed principal component analysis.
\newblock In \emph{Advances in Neural Information Processing Systems}, pages
  2654--2662, 2011.

\bibitem[Buntine(2002)]{buntine02}
W.~Buntine.
\newblock Variational extensions to {EM} and multinomial {PCA}.
\newblock In \emph{Proceedings of the European Conference on Machine Learning},
  2002.

\bibitem[Burges(2010)]{burges2010}
C.~J.~C. Burges.
\newblock Dimension reduction: a guided tour.
\newblock \emph{Foundations \& Trends in Machine Learning}, 2\penalty0
  (4):\penalty0 275--365, 2010.

\bibitem[Candes et~al.(2011)Candes, Li, Ma, and Wright]{candes2011}
E.~J. Candes, X.~Li, Y.~Ma, and J.~Wright.
\newblock Robust principal component analysis?
\newblock \emph{Journal of the ACM}, 58\penalty0 (3):\penalty0 11:1--11:37,
  2011.

\bibitem[Chechik et~al.(2009)Chechik, Shalit, Sharma, and
  Bengio]{chechik2009online}
G.~Chechik, U.~Shalit, V.~Sharma, and S.~Bengio.
\newblock An online algorithm for large scale image similarity learning.
\newblock In \emph{Advances in Neural Information Processing Systems}, pages
  306--314, 2009.

\bibitem[Choulakian(2006)]{choulakian2006}
V.~Choulakian.
\newblock L1-norm projection pursuit principal component analysis.
\newblock \emph{Computational Statistics and Data Analysis}, 50\penalty0
  (6):\penalty0 1441--1451, 2006.

\bibitem[Churchland et~al.(2012)Churchland, Cunningham, Kaufman, Foster,
  Nuyujukian, Ryu, and Shenoy]{churchland2012}
M.~M. Churchland, J.~P. Cunningham, M.~T. Kaufman, J.~D. Foster, P.~Nuyujukian,
  S.~I. Ryu, and K.~V. Shenoy.
\newblock Neural population dynamics during reaching.
\newblock \emph{Nature}, 487:\penalty0 51--56, 2012.

\bibitem[Coifman and Lafon(2006)]{coifman2006}
R.~R. Coifman and S.~Lafon.
\newblock Diffusion maps.
\newblock \emph{Applied and Computational Harmonic Analysis}, 21\penalty0
  (1):\penalty0 5--30, 2006.

\bibitem[Collins et~al.(2002)Collins, Dasgupta, and Schapire]{collins2002}
M.~Collins, S.~Dasgupta, and R.~E. Schapire.
\newblock A generalization of principal component analysis to the exponential
  family.
\newblock In \emph{Advances in Neural Information Processing Systems}, 2002.

\bibitem[Cox and Cox(2001)]{coxBook}
T.~F. Cox and M.~A. Cox.
\newblock \emph{Multidimensional scaling}, volume~88.
\newblock CRC Press, 2001.

\bibitem[Cunningham and Yu(2014)]{cunningham2014dimensionality}
J.~P. Cunningham and B.~M. Yu.
\newblock Dimensionality reduction for large-scale neural recordings.
\newblock \emph{Nature Neuroscience}, 17:\penalty0 1500--1509, 2014.

\bibitem[d'{A}spremont et~al.(2007)d'{A}spremont, Ghaoui, Jordan, and
  Lanckriet]{daspremont2007}
A~d'{A}spremont, L.~El Ghaoui, M.~I. Jordan, and G.~R. Lanckriet.
\newblock A direct formulation for sparse {PCA} using semidefinite programming.
\newblock \emph{{SIAM} Review}, 49:\penalty0 434--448, 2007.

\bibitem[d'{A}spremont et~al.(2008)d'{A}spremont, Bach, and
  Ghaoui]{daspremont2008}
A~d'{A}spremont, F.~R. Bach, and L.~El Ghaoui.
\newblock Optimal solutions for sparse principal component analysis.
\newblock \emph{Journal of Machine Learning Research}, 9:\penalty0 1269--1294,
  2008.

\bibitem[De~la Torre(2012)]{delatorre2012}
F.~De~la Torre.
\newblock A least-squares framework for component analysis.
\newblock \emph{IEEE Transactions on Pattern Analysis and Machine
  Intelligence}, 34\penalty0 (6):\penalty0 1041--1055, 2012.

\bibitem[De~Ridder et~al.(2002)De~Ridder, Duin, and Kittler]{de2002texture}
D.~De~Ridder, R.~P.~W. Duin, and J.~Kittler.
\newblock Texture description by independent components.
\newblock In \emph{Structural, Syntactic, and Statistical Pattern Recognition},
  pages 587--596. Springer, 2002.

\bibitem[Deerwester et~al.(1990)Deerwester, Dumais, Landauer, Furnas, and
  Harshman]{deerwester1990}
S.~C. Deerwester, S.~T. Dumais, T.~K. Landauer, G.~W. Furnas, and R.~A.
  Harshman.
\newblock Indexing by latent semantic analysis.
\newblock \emph{Journal of the American Society for Information Science},
  41\penalty0 (6):\penalty0 391--407, 1990.

\bibitem[Dempster et~al.(1977)Dempster, Laird, and Rubin]{dempster77}
A.~P. Dempster, N.~M. Laird, and D.~B. Rubin.
\newblock Maximum likelihood from incomplete data via the {EM} algorithm (with
  discussion).
\newblock \emph{Journal of the Royal Statistical Society, Series B},
  39:\penalty0 1--38, 1977.

\bibitem[Der and Saul(2012)]{der2012latent}
M.~Der and L.~K. Saul.
\newblock Latent coincidence analysis: a hidden variable model for distance
  metric learning.
\newblock In \emph{Advances in Neural Information Processing Systems}, pages
  3230--3238, 2012.

\bibitem[Diaconis and Freedman(1984)]{diaconis1984asymptotics}
P.~Diaconis and D.~Freedman.
\newblock Asymptotics of graphical projection pursuit.
\newblock \emph{The Annals of Statistics}, pages 793--815, 1984.

\bibitem[Eckart and Young(1936)]{eckart1936}
C.~Eckart and G.~Young.
\newblock The approximation of one matrix by another of lower rank.
\newblock \emph{Psychometrika}, 1:\penalty0 211--218, 1936.

\bibitem[Edelman et~al.(1998)Edelman, Arias, and Smith]{edelman1998}
A.~Edelman, T.~A. Arias, and S.~T. Smith.
\newblock The geometry of algorithms with orthogonality constraints.
\newblock \emph{SIAM Journal of Matrix Analysis and Applications}, 1998.

\bibitem[Fan and Hoffman(1955)]{fan1955}
K.~Fan and A.~J. Hoffman.
\newblock Some metric inequalities in the space of matrices.
\newblock \emph{Proceedings of the American Mathematical Society}, 6:\penalty0
  111--116, 1955.

\bibitem[Fiori(2005)]{fiori2005}
S.~Fiori.
\newblock Quasi-geodesic neural learning algorithms over the orthogonal group:
  a tutorial.
\newblock \emph{Journal of Machine Learning Research}, 6:\penalty0 743--781,
  2005.

\bibitem[Fisher(1936)]{fisher1936}
R.~A. Fisher.
\newblock The use of multiple measurements in taxonomic problems.
\newblock \emph{Annals of Eugenics}, 7\penalty0 (2):\penalty0 179--188, 1936.

\bibitem[Friedman(1987)]{friedman1987exploratory}
J.~H. Friedman.
\newblock Exploratory projection pursuit.
\newblock \emph{Journal of the American Statistical Association}, 82\penalty0
  (397):\penalty0 249--266, 1987.

\bibitem[Fukumizu et~al.(2004)Fukumizu, Bach, and
  Jordan]{fukumizu2004dimensionality}
K.~Fukumizu, F.~R. Bach, and M.~I. Jordan.
\newblock Dimensionality reduction for supervised learning with reproducing
  kernel hilbert spaces.
\newblock \emph{Journal of Machine Learning Research}, 5:\penalty0 73--99,
  2004.

\bibitem[Fukumizu et~al.(2009)Fukumizu, Bach, and Jordan]{fukumizu2009kernel}
K.~Fukumizu, F.~R. Bach, and M.~I. Jordan.
\newblock Kernel dimension reduction in regression.
\newblock \emph{The Annals of Statistics}, pages 1871--1905, 2009.

\bibitem[Fukunaga(1990)]{fukunagaBook}
K.~Fukunaga.
\newblock \emph{Introduction to Statistical Pattern Recognition}.
\newblock Academic press, 1990.

\bibitem[Gabay(1982)]{gabay1982}
D.~Gabay.
\newblock Minimizing a differentiable function over a differentiable manifold.
\newblock \emph{The Journal of Optimization Theory and Applications},
  37\penalty0 (2):\penalty0 177--219, 1982.

\bibitem[Galpin and Hawkins(1987)]{galpin1987}
J.~S. Galpin and D.~M. Hawkins.
\newblock Methods of {L}1 estimation of a covariance matrix.
\newblock \emph{Computational Statistics and Data Analysis}, 5:\penalty0
  305--319, 1987.

\bibitem[Globerson and Roweis(2005)]{globerson2005metric}
A.~Globerson and S.~T. Roweis.
\newblock Metric learning by collapsing classes.
\newblock In \emph{Advances in Neural Information Processing Systems}, pages
  451--458, 2005.

\bibitem[Goldberger et~al.(2004)Goldberger, Roweis, Hinton, and
  Salakhutdinov]{goldberger2004neighbourhood}
J.~Goldberger, S.~T. Roweis, G.~Hinton, and R.~Salakhutdinov.
\newblock Neighbourhood components analysis.
\newblock In \emph{Advances in Neural Information Processing Systems}, 2004.

\bibitem[Golub and {Van Loan}(1996)]{golubBook}
G.~H. Golub and C.~F. {Van Loan}.
\newblock \emph{Matrix Computations, 3rd edition}.
\newblock Hopkins University Press, Baltimore, 1996.

\bibitem[Gretton et~al.(2005)Gretton, Bousquet, Smola, and
  Sch{\"o}lkopf]{gretton2005measuring}
A.~Gretton, O.~Bousquet, A.~Smola, and B.~Sch{\"o}lkopf.
\newblock Measuring statistical dependence with {H}ilbert-{S}chmidt norms.
\newblock In \emph{Algorithmic Learning Theory}, pages 63--77. Springer, 2005.

\bibitem[Gretton et~al.(2012)Gretton, Borgwardt, Rasch, Sch{\"o}lkopf, and
  Smola]{gretton2012kernel}
A.~Gretton, K.~M. Borgwardt, M.~J. Rasch, B.~Sch{\"o}lkopf, and A.~Smola.
\newblock A kernel two-sample test.
\newblock \emph{Journal of Machine Learning Research}, 13\penalty0
  (1):\penalty0 723--773, 2012.

\bibitem[Hardoon and Shawe-Taylor(2009)]{hardoon2009convergence}
D.~R. Hardoon and J.~Shawe-Taylor.
\newblock Convergence analysis of kernel canonical correlation analysis: theory
  and practice.
\newblock \emph{Machine Learning}, 74\penalty0 (1):\penalty0 23--38, 2009.

\bibitem[Hardoon et~al.(2004)Hardoon, Szedmak, and {Shawe-Taylor}]{hardoon2004}
D.~R. Hardoon, S.~Szedmak, and J~{Shawe-Taylor}.
\newblock Canonical correlation analysis: an overview with application to
  learning methods.
\newblock \emph{Neural Computation}, 16\penalty0 (12):\penalty0 2639--2664,
  2004.

\bibitem[Hastie et~al.(2008)Hastie, Tibshirani, and Friedman]{hastieBook}
T.~Hastie, R.~Tibshirani, and J.~Friedman.
\newblock \emph{Elements of Statistical Learning, 2nd Edition}.
\newblock Cambridge University Press, Cambridge, UK, 2008.

\bibitem[He and Niyogi(2004)]{niyogi2004locality}
X.~He and P.~Niyogi.
\newblock Locality preserving projections.
\newblock In \emph{Advances in Neural Information Processing Systems},
  volume~16, page 153, 2004.

\bibitem[He et~al.(2005)He, Cai, Yan, and Zhang]{he2005neighborhood}
X.~He, D.~Cai, S.~Yan, and H.~Zhang.
\newblock Neighborhood preserving embedding.
\newblock In \emph{IEEE International Conference on Computer Vision}, volume~2,
  pages 1208--1213, 2005.

\bibitem[Higham(1989)]{higham1989}
N.~J. Higham.
\newblock Matrix nearness problems and applications.
\newblock In \emph{Applications of Matrix Theory}, pages 1--27. Oxford
  University Press, 1989.

\bibitem[Hotelling(1936)]{hotelling1936}
H.~Hotelling.
\newblock Relations between two sets of variates.
\newblock \emph{Biometrika}, 28:\penalty0 321--377, 1936.

\bibitem[Huber(1985)]{huber1985projection}
P.~J. Huber.
\newblock Projection pursuit.
\newblock \emph{The Annals of Statistics}, pages 435--475, 1985.

\bibitem[Hyvarinen et~al.(2001)Hyvarinen, Karhunen, and Oja]{icaBook}
A.~Hyvarinen, J.~Karhunen, and E.~Oja.
\newblock \emph{Independent Component Analysis}.
\newblock John Wiley and Sons, 2001.

\bibitem[Joho et~al.(2000)Joho, Mathis, and Lambert]{joho2000}
M.~Joho, H.~Mathis, and R.~H. Lambert.
\newblock Overdetermined blind source separation: Using more sensors than
  source signals in a noisy mixture.
\newblock In \emph{Independent Component Analysis and Blind Signal Separation},
  pages 81--86, 2000.

\bibitem[Journee et~al.(2010)Journee, Nesterov, Richtarik, and
  Sepulchre]{journee2010}
M.~Journee, Y.~Nesterov, P.~Richtarik, and R.~Sepulchre.
\newblock Generalized power method for sparse principal component analysis.
\newblock \emph{Journal of Machine Learning Research}, 11:\penalty0 517--553,
  2010.

\bibitem[Kalman(1960)]{kalman60}
R.~E. Kalman.
\newblock A new approach to linear filtering and prediction problems.
\newblock \emph{The Journal of Basic Engineering}, 82:\penalty0 35--45, 1960.

\bibitem[Kaneko et~al.(2013)Kaneko, Fiori, and Tanaka]{kaneko2013empirical}
T.~Kaneko, S.~Fiori, and T.~Tanaka.
\newblock Empirical arithmetic averaging over the compact stiefel manifold.
\newblock \emph{IEEE Transactions on Signal Processing}, 61\penalty0
  (4):\penalty0 883--894, 2013.

\bibitem[Kao and {Van Roy}(2013)]{vanroy2011}
Y.~H. Kao and B.~{Van Roy}.
\newblock Learning a factor model via regularized {PCA}.
\newblock \emph{Machine Learning}, 91\penalty0 (279-303), 2013.

\bibitem[Kulis(2012)]{kulis2012metric}
B.~Kulis.
\newblock Metric learning: A survey.
\newblock \emph{Foundations \& Trends in Machine Learning}, 5\penalty0
  (4):\penalty0 287--364, 2012.

\bibitem[Larsen(2002)]{larsen2002}
R.~Larsen.
\newblock Decomposition using maximum autocorrelation factors.
\newblock \emph{Journal of Chemometrics}, 16:\penalty0 427--435, 2002.

\bibitem[Lawrence(2012)]{lawrence2012unifying}
N.~D. Lawrence.
\newblock A unifying probabilistic perspective for spectral dimensionality
  reduction: insights and new models.
\newblock \emph{Journal of Machine Learning Research}, 13\penalty0
  (1):\penalty0 1609--1638, 2012.

\bibitem[Lee and Seung(1999)]{lee1999}
D.~D. Lee and H.~S. Seung.
\newblock Learning the parts of objects by non-negative matrix factorization.
\newblock \emph{Nature}, 401:\penalty0 788--791, 1999.

\bibitem[Li and Hu(2011)]{li2011}
J.~F. Li and X.~Y. Hu.
\newblock {Procrustes problems and associated approximation problems for
  matrices with $k$-involutory symmetries.}
\newblock \emph{Linear Algebra and its Applications}, 434:\penalty0 820--829,
  2011.

\bibitem[Li(1991)]{li1991sliced}
K.~C. Li.
\newblock Sliced inverse regression for dimension reduction.
\newblock \emph{Journal of the American Statistical Association}, 86\penalty0
  (414):\penalty0 316--327, 1991.

\bibitem[Luenberger(1972)]{luenberger1972}
D.~Luenberger.
\newblock The gradient projection method along geodesics.
\newblock \emph{Management Science}, 18\penalty0 (11), 1972.

\bibitem[Manton(2002)]{manton2002}
J.~H. Manton.
\newblock Optimization algorithms exploiting unitary constraints.
\newblock \emph{IEEE Transactions on Signal Processing}, 50:\penalty0 635--650,
  2002.

\bibitem[Manton(2004)]{manton2004}
J.~H. Manton.
\newblock On the various generalizations of optimization algorithms to
  manifolds.
\newblock \emph{Proceedings of Mathematical Theory of Network and Systems},
  2004.

\bibitem[Mardia et~al.(1979)Mardia, Kent, and Bibby]{mardiaBook}
K.~V. Mardia, J.~T. Kent, and J.~M. Bibby.
\newblock \emph{Multivariate Analysis}.
\newblock Probability and mathematical statistics. Academic Press, 1979.

\bibitem[Mika et~al.(1999)Mika, R\"{a}tsch, Weston, Sch\"{o}lkopf, and
  M\"{u}ller]{mika1999}
S.~Mika, G.~R\"{a}tsch, J.~Weston, B.~Sch\"{o}lkopf, and K.~R. M\"{u}ller.
\newblock Fisher discriminant analysis with kernels.
\newblock In \emph{Proceedings of the IEEE Signal Processing Society}, pages
  41--48, 1999.

\bibitem[Mirsky(1960)]{mirsky1960}
L.~Mirsky.
\newblock Symmetric gauge functions and unitarily invariant norms.
\newblock \emph{Quarterly Journal of Mathematics}, 11:\penalty0 80--89, 1960.

\bibitem[Mnih and Salakhutdinov(2007)]{mnih2007}
A.~Mnih and R.~Salakhutdinov.
\newblock Probabilistic matrix factorization.
\newblock In \emph{Advances in Neural Information Processing Systems}, pages
  1257--1264, 2007.

\bibitem[Mohamed et~al.(2008)Mohamed, Heller, and Ghahramani]{shakir2008}
S.~Mohamed, K.~Heller, and Z.~Ghahramani.
\newblock Bayesian exponential family {PCA}.
\newblock In \emph{Advances in Neural Information Processing Systems}, 2008.

\bibitem[Muirhead(2005)]{muirheadBook}
R.~J. Muirhead.
\newblock \emph{Aspects of Multivariate Statistical Theory, 2nd Edition}.
\newblock Wiley, 2005.

\bibitem[Nilsson et~al.(2007)Nilsson, Sha, and Jordan]{nilsson2007regression}
J.~Nilsson, F.~Sha, and M.~I. Jordan.
\newblock Regression on manifolds using kernel dimension reduction.
\newblock In \emph{Proceedings of the International Conference on Machine
  Learning}, pages 697--704, 2007.

\bibitem[Nishimori and Akaho(2005)]{nishimori2005}
Y.~Nishimori and S.~Akaho.
\newblock Learning algorithms utilizing quasi-geodesic flows on the {Stiefel}
  manifold.
\newblock \emph{Neurocomputing}, 67:\penalty0 106--135, 2005.

\bibitem[Pearson(1901)]{pearson1901}
K.~Pearson.
\newblock On lines and planes of closest fit to systems of points in space.
\newblock \emph{Philosophical Magazine}, 2:\penalty0 559--572, 1901.

\bibitem[Peltonen et~al.(2007)Peltonen, Goldberger, and
  Kaski]{peltonen2007fast}
J.~Peltonen, J.~Goldberger, and S.~Kaski.
\newblock Fast semi-supervised discriminative component analysis.
\newblock In \emph{Proceeding of the IEEE Workshop on Machine Learning for
  Signal Processing}, pages 312--317. IEEE, 2007.

\bibitem[Rao(1948)]{rao1948}
C.~R. Rao.
\newblock The utilization of multiple measurements in problems of biological
  classification.
\newblock \emph{Journal of the Royal Statistical Society, Series B},
  10\penalty0 (2):\penalty0 159--203, 1948.

\bibitem[Rasmussen and Williams(2006)]{rasmussenBook}
C.~E. Rasmussen and C.K.I. Williams.
\newblock \emph{Gaussian Processes for Machine Learning}.
\newblock MIT Press, Cambridge, 2006.

\bibitem[Rennie and Srebro(2005)]{rennie2005}
J.~D. Rennie and N.~Srebro.
\newblock Fast maximum margin matrix factorization for collaborative
  prediction.
\newblock In \emph{Proceedings of the International Conference on Machine
  Learning}, pages 713--719, 2005.

\bibitem[Roweis(1997)]{roweisSPCA}
S.~T. Roweis.
\newblock {EM} algorithms for {PCA} and sensible {PCA}.
\newblock In \emph{Advances in Neural Information Processing Systems}, 1997.

\bibitem[Roweis and Saul(2000)]{roweis2000}
S.~T. Roweis and L.~K. Saul.
\newblock Nonlinear dimensionality reduction by locally linear embedding.
\newblock \emph{Science}, 290\penalty0 (5500):\penalty0 2323--2326, December
  2000.

\bibitem[Rubinshtein and Srivastava(2010)]{srivastava2010}
E.~Rubinshtein and A.~Srivastava.
\newblock Optimal linear projections for enhancing desired data statistics.
\newblock \emph{Statistical Computing}, 20:\penalty0 267--282, 2010.

\bibitem[Ruhe(1987)]{ruhe1987}
A.~Ruhe.
\newblock Closest normal matrix finally found!
\newblock \emph{Technical Report, University of Goteberg}, 1987.

\bibitem[Sch\"{o}lkopf et~al.(1999)Sch\"{o}lkopf, Smola, and Muller]{kpca}
B.~Sch\"{o}lkopf, A.~Smola, and R.~K. Muller.
\newblock {Kernel principal component analysis}.
\newblock \emph{Advances in Kernel Methods: Support Vector Learning}, pages
  327--352, 1999.

\bibitem[Schonemann(1966)]{schonemann1966}
P.~H. Schonemann.
\newblock A generalized solution of the orthogonal {Procrustes} problem.
\newblock \emph{Psychometrika}, 31:\penalty0 1--10, 1966.

\bibitem[Shen et~al.(2007)Shen, Li, and Brooks]{shen2007}
C.~Shen, H.~Li, and M.~J. Brooks.
\newblock A convex programming approach to the trace quotient problem.
\newblock In \emph{Proceedings of the Asian Conference on Computer Vision},
  pages 227--235. Springer, 2007.

\bibitem[Spearman(1904)]{spearman1904}
C.~Spearman.
\newblock General intelligence, objectively determined and measured.
\newblock \emph{American Journal of Psychology}, 15:\penalty0 201--293, 1904.

\bibitem[Srebro and Jaakkola(2003)]{srebro2003}
N.~Srebro and T.~S. Jaakkola.
\newblock Weighted low-rank approximations.
\newblock \emph{Proceedings of the International Conference on Machine
  Learning}, pages 720--727, 2003.

\bibitem[Srebro et~al.(2004)Srebro, Rennie, and Jaakkola]{srebro2004}
N.~Srebro, J.~Rennie, and T.~S. Jaakkola.
\newblock Maximum-margin matrix factorization.
\newblock In \emph{Advances in Neural Information Processing Systems}, pages
  1329--1336, 2004.

\bibitem[Srivastava and Liu(2005)]{srivastava2005}
A.~Srivastava and X.~Liu.
\newblock Tools for application-driven linear dimension reduction.
\newblock \emph{Neurocomputing}, 67:\penalty0 136--160, 2005.

\bibitem[Stone and Porrill(1998)]{stone1998}
J.~V. Stone and J.~Porrill.
\newblock Undercomplete independent component analysis for signal separation
  and dimension reduction.
\newblock \emph{Technical Report}, 1998.

\bibitem[Stone and Brooks(1990)]{stone1990continuum}
M.~Stone and R.~J. Brooks.
\newblock Continuum regression: cross-validated sequentially constructed
  prediction embracing ordinary least squares, partial least squares and
  principal components regression.
\newblock \emph{Journal of the Royal Statistical Society, Series B}, pages
  237--269, 1990.

\bibitem[Sun et~al.(2009)Sun, Ji, and Ye]{sun2009}
L.~Sun, S.~Ji, and J.~Ye.
\newblock A least squares formulation for a class of generalized eigenvalue
  problems in machine learning.
\newblock In \emph{Proceedings of the International Conference on Machine
  Learning}, pages 977--984. ACM, 2009.

\bibitem[Switzer and Green(1984)]{switzer1984}
P.~Switzer and A.~A. Green.
\newblock Min/max autocorrelation factors for multivariate spatial imagery.
\newblock \emph{Technical Report, Stanford University}, 1984.

\bibitem[Tenenbaum et~al.(2000)Tenenbaum, de{Silva}, and
  Langford]{tenenbaum2000}
J.~B. Tenenbaum, V.~de{Silva}, and J.~C. Langford.
\newblock A global geometric framework for nonlinear dimensionality reduction.
\newblock \emph{Science}, 290\penalty0 (5500):\penalty0 2319--2323, December
  2000.

\bibitem[Theobald(1975)]{theobald75}
C.~M. Theobald.
\newblock An inequality with application to multivariate analysis.
\newblock \emph{Biometrika}, 62\penalty0 (2):\penalty0 461--466, 1975.

\bibitem[Tibshirani(1996)]{tibshirani1996}
R.~Tibshirani.
\newblock Regression shrinkage and selection via the lasso.
\newblock \emph{Journal of the Royal Statistical Society, Series B},
  58:\penalty0 267--288, 1996.

\bibitem[Timm(2002)]{timm2002applied}
N.~H. Timm.
\newblock \emph{Applied Multivariate Analysis}.
\newblock Springer, 2002.

\bibitem[Tipping and Bishop(1999)]{tipping99}
M.~E. Tipping and C.~M. Bishop.
\newblock Probabilistic principal component analysis.
\newblock \emph{Journal of the Royal Statistical Society, Series B},
  61\penalty0 (3):\penalty0 611--622, 1999.

\bibitem[Torgerson(1952)]{torgerson1952}
W.~S. Torgerson.
\newblock Multidimensional scaling: I. theory and method.
\newblock \emph{Psychometrika}, 17\penalty0 (4):\penalty0 401--419, 1952.

\bibitem[Torresani and Lee(2006)]{torresani2006large}
L.~Torresani and K.~Lee.
\newblock Large margin component analysis.
\newblock In \emph{Advances in Neural Information Processing Systems}, pages
  1385--1392, 2006.

\bibitem[Turner and Sahani(2007)]{turner2007}
R.~Turner and M.~Sahani.
\newblock A maximum-likelihood interpretation for slow feature analysis.
\newblock \emph{Neural Computation}, 19\penalty0 (4):\penalty0 1022--1038,
  2007.

\bibitem[Ulfarsson and Solo(2008)]{ulfarsson2008}
M.~O. Ulfarsson and V.~Solo.
\newblock Sparse variable {PCA} using geodesic steepest descent.
\newblock \emph{IEEE Transactions on Signal Processing}, 56:\penalty0
  5823--5832, 2008.

\bibitem[{Van der Maaten} et~al.(2009){Van der Maaten}, Postma, and {Van den
  Herik}]{van2009}
L.~J.~P. {Van der Maaten}, E.~O. Postma, and H.~J. {Van den Herik}.
\newblock Dimensionality reduction: A comparative review.
\newblock \emph{Tilburg University Technical Report, TiCC-TR 2009-005}, 2009.

\bibitem[Van~der Maaten and Hinton(2008)]{van2008visualizing}
Laurens Van~der Maaten and Geoffrey Hinton.
\newblock Visualizing data using t-{SNE}.
\newblock \emph{Journal of Machine Learning Research}, 9\penalty0
  (2579-2605):\penalty0 85, 2008.

\bibitem[Varshney and Willsky(2011)]{varshney2011}
K.~R. Varshney and A.~S. Willsky.
\newblock Linear dimensionality reduction for margin-based classification:
  high-dimensional data and sensor networks.
\newblock \emph{IEEE Transactions on Signal Processing}, 59:\penalty0
  2496--2512, 2011.

\bibitem[Wang et~al.(2010)Wang, Sha, and Jordan]{wang2010unsupervised}
M.~Wang, F.~Sha, and M.~I. Jordan.
\newblock Unsupervised kernel dimension reduction.
\newblock In \emph{Advances in Neural Information Processing Systems}, pages
  2379--2387, 2010.

\bibitem[Weinberger and Saul(2006)]{weinberger2006}
K.~Q. Weinberger and L.~K. Saul.
\newblock Unsupervised learning of image manifolds by semidefinite programming.
\newblock \emph{International Journal of Computer Vision}, 70\penalty0
  (1):\penalty0 77--90, 2006.

\bibitem[Weinberger and Saul(2009)]{weinberger2009distance}
K.~Q. Weinberger and L.~K. Saul.
\newblock Distance metric learning for large margin nearest neighbor
  classification.
\newblock \emph{Journal of Machine Learning Research}, 10:\penalty0 207--244,
  2009.

\bibitem[Weinberger et~al.(2005)Weinberger, Blitzer, and
  Saul]{weinberger2005distance}
K.~Q. Weinberger, J.~Blitzer, and L.~K. Saul.
\newblock Distance metric learning for large margin nearest neighbor
  classification.
\newblock In \emph{Advances in Neural Information Processing Systems}, pages
  1473--1480, 2005.

\bibitem[Welling et~al.(2003)Welling, Agakov, and Williams]{welling2003}
M.~Welling, F.~Agakov, and C.~K.~I. Williams.
\newblock Extreme component analysis.
\newblock In \emph{Advances in Neural Information Processing Systems}, 2003.

\bibitem[Welling et~al.(2004)Welling, Zemel, and Hinton]{welling2004}
M.~Welling, R.~S. Zemel, and G.~E. Hinton.
\newblock Probabilistic sequential independent components analysis.
\newblock In \emph{IEEE Transactions on Neural Networks}, 2004.

\bibitem[Williams(2002)]{williams2002}
C.~K.~I. Williams.
\newblock On a connection between kernel {PCA} and metric multidimensional
  scaling.
\newblock \emph{Machine Learning}, 46\penalty0 (1-3):\penalty0 11--19, 2002.

\bibitem[Williams and Agakov(2002)]{agakov2002}
C.~K.~I. Williams and F.~Agakov.
\newblock Products of {G}aussians and probabilistic minor component analysis.
\newblock \emph{Neural Computation}, 14\penalty0 (5):\penalty0 1169--1182,
  2002.

\bibitem[Wiskott(2003)]{wiskott2003}
L.~Wiskott.
\newblock Slow feature analysis: A theoretical analysis of optimal free
  responses.
\newblock \emph{Neural Computation}, 15\penalty0 (9):\penalty0 2147--2177,
  2003.

\bibitem[Wiskott and Sejnowski(2002)]{wiskott2002}
L.~Wiskott and T.~Sejnowski.
\newblock Slow feature analysis: unsupervised learning of invariances.
\newblock \emph{Neural Computation}, 14\penalty0 (4):\penalty0 715--770, 2002.

\bibitem[Xing et~al.(2002)Xing, Jordan, Russell, and Ng]{xing2002distance}
E.~P. Xing, M.~I. Jordan, S.~Russell, and A.~Y. Ng.
\newblock Distance metric learning with application to clustering with
  side-information.
\newblock In \emph{Advances in Neural Information Processing Systems}, pages
  505--512, 2002.

\bibitem[Yan and Tang(2006)]{yan2006}
S.~Yan and X.~Tang.
\newblock Trace quotient problems revisited.
\newblock In \emph{Proceedings of the European Conference on Computer Vision},
  pages 232--244. Springer, 2006.

\bibitem[Yang(2007)]{yang2007overview}
L.~Yang.
\newblock An overview of distance metric learning.
\newblock \emph{Proceedings of Computer Vision and Pattern Recognition}, 7,
  2007.

\bibitem[Yang and Jin(2006)]{yang2006distance}
L.~Yang and R.~Jin.
\newblock Distance metric learning: A comprehensive survey.
\newblock \emph{Michigan State Universiy Technical Report}, 2006.

\bibitem[Yu et~al.(2006)Yu, Yu, Tresp, Kriegel, and Wu]{yu2006supervised}
S.~Yu, K.~Yu, V.~Tresp, H.~P. Kriegel, and M.~Wu.
\newblock Supervised probabilistic principal component analysis.
\newblock In \emph{Proceedings of the International Conference on Knowledge
  Discovery and Data Mining}, pages 464--473, 2006.

\bibitem[Zhang et~al.(1999)Zhang, Cichocki, and Amari]{zhang1999}
L.~Zhang, A.~Cichocki, and S.~Amari.
\newblock Natural gradient algorithm for blind separation of overdetermined
  mixture with additive noise.
\newblock \emph{IEEE Signal Processing Letters}, 6\penalty0 (11):\penalty0
  293--295, 1999.

\bibitem[Zhao et~al.(2007)Zhao, Lin, and Tang]{zhao2007}
D.~Zhao, Z.~Lin, and X.~Tang.
\newblock Laplacian {PCA} and its applications.
\newblock In \emph{Proceedings of the International Conference on Computer
  Vision}, pages 1--8, 2007.

\bibitem[Zou et~al.(2006)Zou, Hastie, and Tibshirani]{zou2006}
H.~Zou, T.~Hastie, and R.~Tibshirani.
\newblock Sparse principal component analysis.
\newblock \emph{Journal of Computational and Graphical Statistics}, 15\penalty0
  (2):\penalty0 265--286, 2006.

\end{thebibliography}

\end{document}